\documentclass[a4paper,fleqn]{cas-dc}

\usepackage[numbers]{natbib}

\def\tsc#1{\csdef{#1}{\textsc{\lowercase{#1}}\xspace}}
\tsc{WGM}
\tsc{QE}

\newtheorem{theorem}{Theorem}
\newproof{proof}{Proof}

\usepackage{algorithm}
\usepackage{algorithmic}

\begin{document}
\let\WriteBookmarks\relax
\def\floatpagepagefraction{1}
\def\textpagefraction{.001}

\shorttitle{CARD: Cyclic Adversarial Robust Distillation}

\shortauthors{Z. Li et al.}

\title [mode = title]{
CARD: Closing the Teacher-Student Loop for Adversarial Robustness Distillation}

\author[1]{Zhi Li}[orcid=0009-0009-7160-7145]
\ead{lizhi@stu.ynu.edu.cn}
\credit{Writing-Original Draft, Methodology, Visualization, Data Curation}

\author[1]{Haowei Liu}[orcid=0009-0006-1944-5941]
\ead{liuhaowei@stu.ynu.edu.cn}
\credit{Investigation, Data Curation}

\author[1]{Hongchen Yang}[orcid=0009-0000-2755-4052]
\ead{yhc@stu.ynu.edu.cn}
\credit{Investigation, Validation}

\author[2]{Xiaoxuan Wang}[orcid=0000-0002-0922-4719]
\ead{wangxiaoxuan@ynnu.edu.cn}
\credit{Conceptualization, Methodology}

\author[1]{Song Gao}[orcid=0000-0002-7169-6370]
\cormark[1]
\ead{gaos@ynu.edu.cn}
\credit{Conceptualization, Methodology, Writing-Original Draft}

\author[1]{Zhenli He}[orcid=0000-0002-7986-2222]
\ead{hezl@ynu.edu.cn}
\credit{Conceptualization, Methodology, Funding Acquisition}

\author[1]{Wei Zhou}[orcid=0000-0002-5881-9436]
\ead{zwei@ynu.edu.cn}
\credit{Funding Acquisition}

\author[1]{Shaowen Yao}[orcid=0000-0003-1516-4246]
\ead{yaosw@ynu.edu.cn}
\credit{Funding Acquisition}

\affiliation[1]{
    organization={Engineering Research Center of Cyberspace and School of Software and AI, Yunnan University},
    city={Kunming},
    country={China}
}
\affiliation[2]{
organization={School of Information Science and Technology, Yunnan Normal University},
city={Kunming},
country={China}
}

\cortext[1]{Corresponding author}

\begin{abstract}
Adversarial robustness distillation typically treats a robust teacher as a fixed knowledge source and transfers robustness to a compact student in a one-way manner. This paradigm terminates teacher-student interaction once the student is distilled, thereby discarding potentially useful feedback from the distilled student. However, exploiting such feedback is nontrivial because indiscriminate weak-to-strong transfer may propagate student errors to the teacher. To address this issue, we propose Cyclic Adversarial Robustness Distillation (CARD), a closed-loop framework that turns the distilled student into a feedback source for refining the robust teacher and subsequently transfers the teacher gains back to the same student. CARD first performs strong-to-weak adversarial distillation and then uses reliability-enhanced student feedback to refine the robust teacher.
The refinement combines inverse adversarial student feedback with teacher-student discrepancy maximization, after which the refined teacher is used to continue training the same student under the unchanged base distillation objective. We further establish a sample-level sufficient condition for reliable weak-to-strong feedback transfer and derive a distribution-level upper bound on the refined teacher's risk. Experiments on CIFAR-10, CIFAR-100, and TinyImageNet-10 across multiple teacher-student architectures and base distillation methods show that CARD improves adversarial robustness in most evaluated settings while generally preserving clean accuracy.
These results show that the distilled student can provide useful feedback to the teacher, enabling robustness transfer to proceed beyond the conventional one-way teacher-to-student paradigm.

\end{abstract}

\begin{highlights}
\item Adversarial distillation is extended from one-way transfer to closed-loop feedback.
\item Reliability-enhanced student feedback guides targeted teacher refinement.
\item Theory characterizes when weak-to-strong feedback preserves correct predictions.
\item Refined-teacher gains improve students without redesigning the base objective.
\end{highlights}

\begin{keywords}
 Adversarial Robustness Distillation\sep Adversarial Robustness\sep Knowledge Distillation\sep Weak-to-Strong Learning\sep Teacher-Student Feedback
\end{keywords}

\maketitle

\section{Introduction}\label{}

Lightweight neural networks are increasingly deployed on mobile, edge, and embedded devices because of their low computational and storage costs~\cite{ad_mtard_eccv2022}.
However, their limited capacity makes it difficult to maintain both clean accuracy and adversarial robustness~\cite{att_CUI2026132185,AD_STARSHIP_eccv2024,AD_smaraad_2024_CVPR}.
Adversarial training improves robustness by optimizing models against adversarial perturbations~\cite{atpgd_ICLRmadry2018}, but its effectiveness often relies on substantial model capacity and computational cost~\cite{AD_iad2022iclr,ad_IBD_neurips2023}.
Adversarial robustness distillation provides a practical alternative by transferring robust knowledge from a high-capacity robust teacher to a compact student~\cite{AD_ard2020aaai}.

Existing adversarial-distillation methods mainly focus on improving the transfer of robust knowledge from teacher to student.
Representative approaches employ robust soft labels~\cite{AD_rslad2021iccv}, adaptive teacher guidance~\cite{AdaAD2023cpvr}, feature or gradient alignment~\cite{ad_gacd_tifs2023,AD_IGDM_2025ICLR}, and adversarial trajectory modeling~\cite{AD_DARWIN_2024_CVPR}.
Despite these advances, most methods still follow a one-way teacher-to-student paradigm in which the teacher remains fixed throughout distillation.
Once distillation is completed, the student becomes the endpoint of the transfer process. This one-way formulation may leave potentially useful teacher-student interactions underexploited. Although the distilled student remains weaker than the robust teacher overall, its predictions can still provide complementary feedback, while teacher-student disagreement can reveal inputs on which their behaviors differ substantially. This observation motivates us to investigate whether feedback from the distilled student can be used to refine a stronger robust teacher. However, weak-to-strong feedback can be risky, as aligning a strong teacher with predictions from a weaker student may propagate student errors and degrade the teacher.
The central challenge is therefore to exploit student feedback while reducing the risk of erroneous knowledge transfer and focusing teacher refinement on informative regions of the input space.

In this paper, we propose Cyclic Adversarial Robustness Distillation (CARD), a closed-loop framework that extends conventional one-way robustness transfer into a teacher-student feedback cycle.
As illustrated in Figure~\ref{fig:motivation}, CARD consists of strong-to-weak adversarial distillation, weak-to-strong teacher refinement, and subsequent strong-to-weak gain transfer.
CARD first trains a compact student under the guidance of a pretrained robust teacher using an existing adversarial-distillation objective.
The distilled student is then frozen and used to refine the teacher through two complementary mechanisms.
For clean-input refinement, CARD constructs inverse adversarial examples that reduce the student's classification loss and uses the resulting student predictions to supervise the teacher on the corresponding clean inputs.
For robust refinement, CARD identifies perturbed inputs by maximizing the discrepancy between the clean student prediction and the teacher prediction on the perturbed input, and then uses the clean student prediction to supervise the teacher on these inputs. After refinement, the teacher is frozen and used to continue training the same student from its previous checkpoint under the unchanged base distillation objective. In this way, CARD reorganizes the teacher-student training process into a closed feedback loop without requiring a redesign of the base distillation loss.

A key question is whether feedback from a weaker student can be used to refine a stronger teacher without compromising correct predictions. We therefore provide theoretical analysis of the weak-to-strong refinement process. At the sample level, we establish a sufficient condition under which aligning the teacher with student feedback preserves the correct prediction. At the distribution level, we further derive an upper bound on the refined teacher's classification risk. These results characterize when weak-to-strong feedback can be transferred reliably and provide theoretical support for the proposed closed-loop formulation.

The main contributions of this work are summarized as follows:
\begin{itemize}
    \item We introduce CARD, a closed-loop adversarial robustness distillation framework in which a distilled student is used to refine the robust teacher, and the refined teacher then guides further training of the same student. CARD retains the original adversarial-distillation objective for student training, allowing it to augment existing methods without redesigning their base losses.

    \item We develop a student-feedback-based teacher-refinement mechanism that improves feedback quality and focuses refinement on regions of teacher-student disagreement. Inverse adversarial examples provide lower-loss student feedback for clean-input refinement, and  discrepancy maximization identifies adversarial inputs for robust refinement. We further establish a sample-level sufficient condition for preserving correct predictions under weak-to-strong feedback and derive a distribution-level upper bound on the refined teacher's classification risk.

    \item We conduct extensive experiments on CIFAR-10, CIFAR-100, and TinyImageNet-10 across multiple teacher-student architectures and base adversarial-distillation methods. CARD improves adversarial robustness over the corresponding base methods in most evaluated settings while generally preserving clean accuracy, and ablation studies further verify the effectiveness of the proposed teacher-refinement and closed-loop mechanisms.
\end{itemize}

\begin{figure}[pos=t]
    \centering
    \includegraphics[width=\columnwidth]{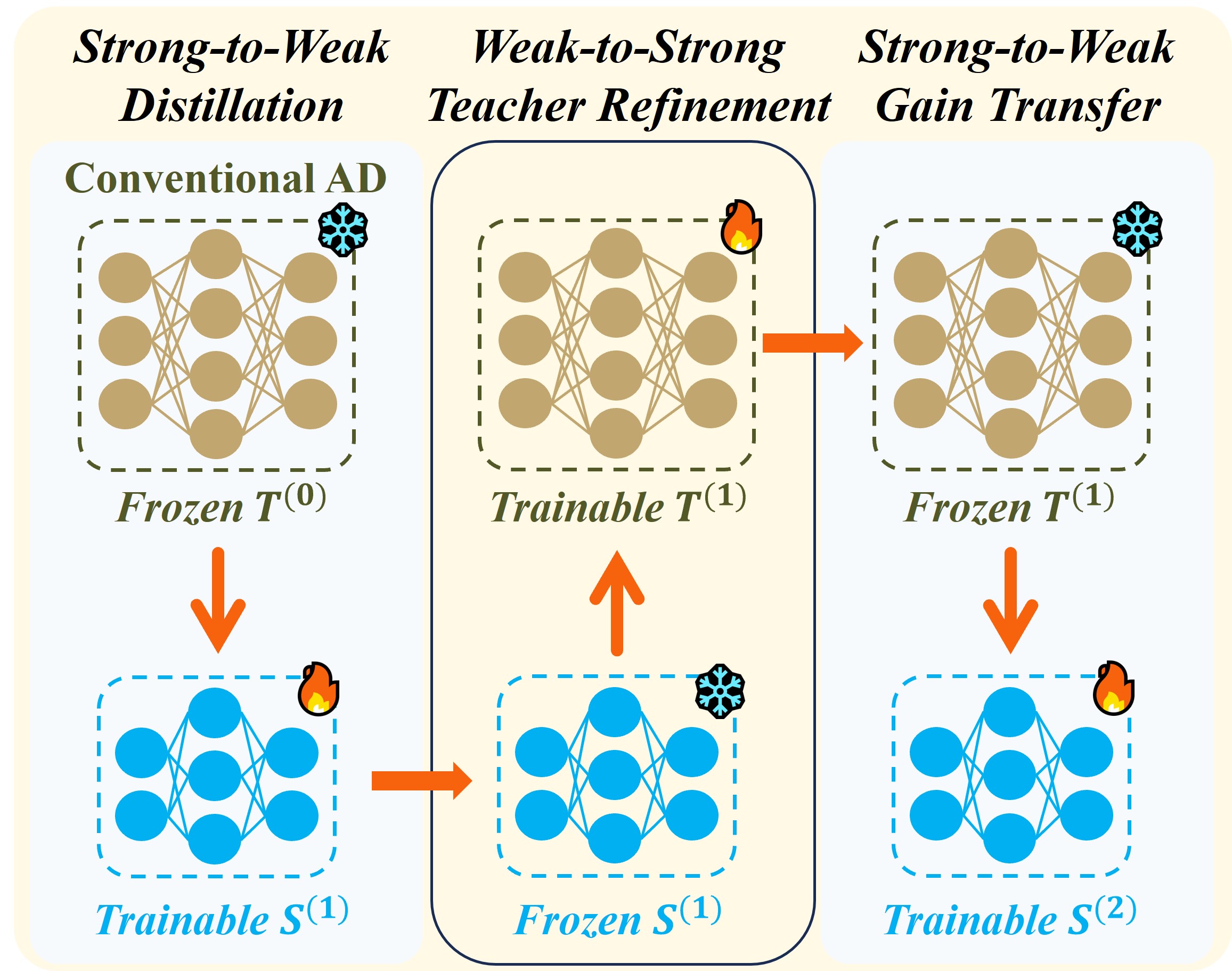}
    \caption{
    From one-way to closed-loop adversarial robustness distillation.
    Conventional adversarial distillation treats the robust teacher as a fixed knowledge source and transfers robustness only from teacher to student.
    CARD closes the teacher-student loop by using the distilled student to refine the robust teacher and then using the refined teacher to continue training the same student.
    Snowflake and flame symbols denote frozen and trainable networks, respectively.
    }
    \label{fig:motivation}
\end{figure}

\begin{figure*}[t]
    \centering
    \includegraphics[width=0.95\textwidth]{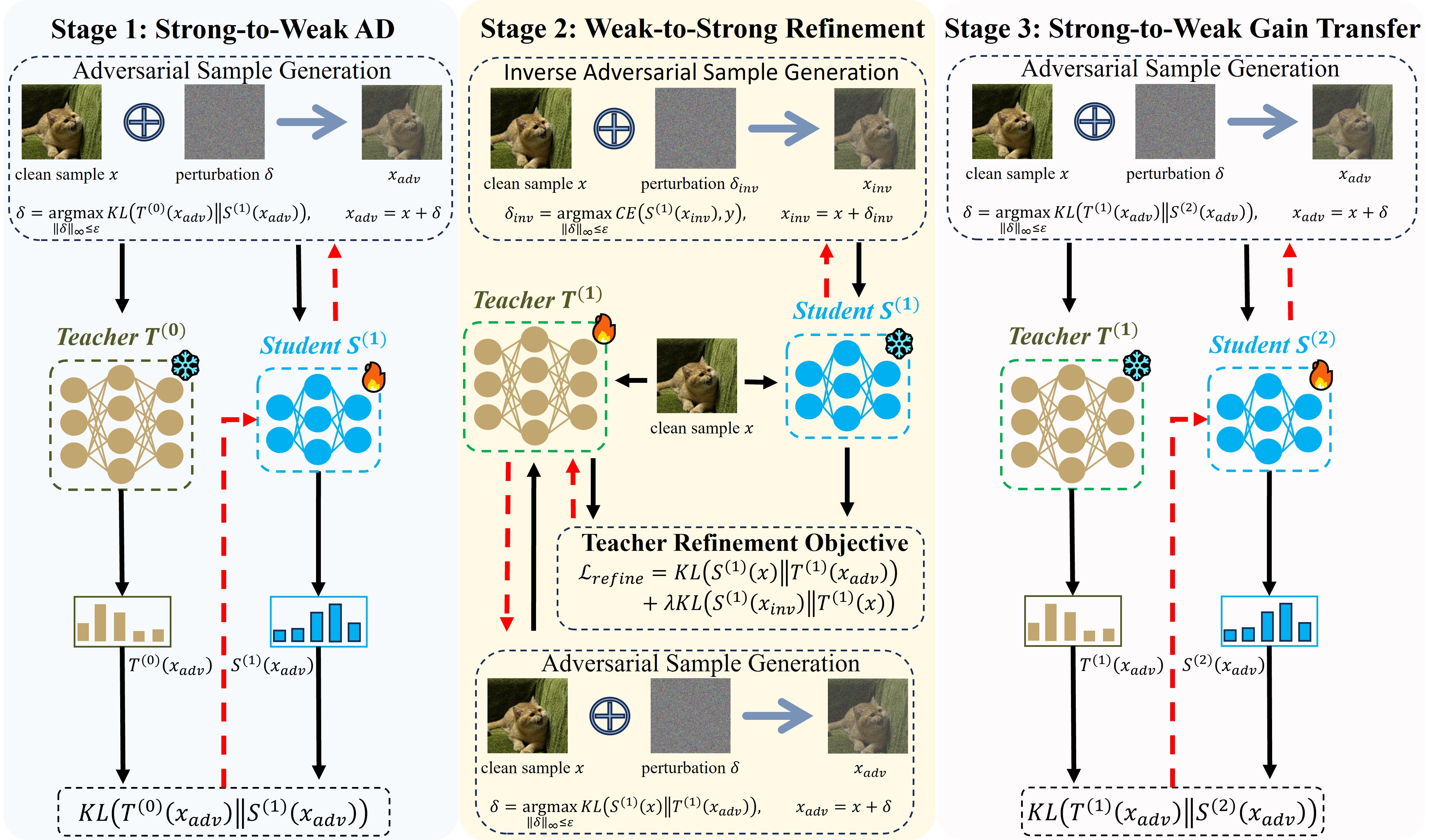}
    \caption{
        Overview of CARD.
        Stage 1 performs strong-to-weak adversarial distillation from a pretrained robust teacher to a compact student.
        Stage 2 freezes the first-round student and uses inverse adversarial examples and teacher adversarial examples to refine the pretrained teacher.
        Stage 3 replaces the initial teacher with the refined teacher and continues training from the first-round student checkpoint, transferring the teacher improvement back to the student.
    }
    \label{fig:framework}
\end{figure*}

\section{Related Work}

In this section, we review the three lines of research most closely related to CARD, including adversarial training, adversarial robustness distillation, and teacher adaptation with weak-to-strong feedback.

\subsection{Adversarial Training}
Adversarial training improves model robustness by training against adversarial examples generated within a prescribed perturbation set.
PGD-AT~\cite{atpgd_ICLRmadry2018} formulated adversarial training as a min--max optimization problem and has become a widely adopted baseline for adversarial robustness.
Subsequent studies have improved adversarial training from different perspectives, including adapting attack difficulty during training~\cite{at_CAT_ijcai2018,at_fat_icml2020,at_dart_icml2019,at_lasat_cvpr2022}, mitigating the trade-off between clean accuracy and adversarial robustness~\cite{at_tradesicml2019,ARREST_ATwithKD_ICCV2023,AR-AT_ICLR2025}, and redesigning the training objective~\cite{at_score_2022_ICML,at_dkl_neurips2024}.
Other approaches further enhance robust learning by exploiting informative training samples and local neighborhood structures~\cite{stat2023iclr,HFAT_2024_CVPR,at_uiat_cvpr2023,at_hat_iclr2022,at_vdat_neurips2025}.
Despite this progress, achieving strong adversarial robustness remains difficult for compact networks with limited representational capacity.
This limitation motivates adversarial robustness distillation, which transfers robust knowledge from a high-capacity teacher to a lightweight student.

\subsection{Adversarial Robustness Distillation}
Adversarial robustness distillation transfers robust knowledge from a pretrained teacher to a compact student, aiming to improve student robustness without increasing its model capacity.
ARD~\cite{AD_ard2020aaai} incorporates robust teacher supervision into the adversarial training of the student.
Subsequent methods improve teacher supervision through robust soft labels or adaptive guidance~\cite{AD_rslad2021iccv,AD_iad2022iclr}, and enrich knowledge transfer through feature, gradient, or adversarial-trajectory alignment~\cite{ad_gacd_tifs2023,AD_IGDM_2025ICLR,AD_DARWIN_2024_CVPR}.
Other approaches further adapt the distillation process to the current student.
PeerAiD~\cite{AD_peeraid2024cvpr} jointly trains an online peer tutor with the student, and AdaAD~\cite{AdaAD2023cpvr} exploits teacher-student discrepancy to generate adaptive adversarial examples for student training.

Despite these advances, existing methods primarily focus on improving knowledge transfer from the teacher to the student.
The distilled student is generally treated as the recipient of robust knowledge rather than as a source of feedback for refining the teacher after distillation.
Although both AdaAD and CARD exploit teacher-student discrepancy, the discrepancy serves different roles in the two frameworks.
AdaAD uses it to generate adaptive adversarial examples for student training, whereas CARD employs discrepancy maximization in a separate weak-to-strong stage to identify perturbed inputs where the frozen student's predictions are used to refine the robust teacher.

\subsection{Teacher Adaptation and Weak-to-Strong Learning}
Several recent adversarial-distillation methods move beyond a strictly fixed supervisory model.
PeerAiD~\cite{AD_peeraid2024cvpr} jointly trains an online peer tutor with the student, and CIARD~\cite{AD_ciard2025iccv} adopts a multi-teacher cyclic framework with clean and robust teachers and iteratively updates the robust teacher during student training.
These methods adapt the supervisory model concurrently with student learning.
CARD considers a different setting in which teacher refinement is performed after the student has completed an initial adversarial-distillation stage.
Freezing the distilled student provides a stable feedback source and decouples weak-to-strong refinement from the evolving student-training dynamics, allowing post-distillation student feedback and teacher-student discrepancies to be explicitly utilized for teacher refinement.
The resulting teacher is then used to provide updated robustness supervision to the same student, so that teacher refinement ultimately serves the further improvement of the compact student.

A related line of work, weak-to-strong learning, investigates whether supervision from a weaker model can improve a stronger model~\cite{w2s_first_icml2024,w2s_theoreticalAnalysis_NEURIPS2024,w2s_advw2s_neurips2026}.
Although this direction is closely related to the teacher-refinement stage of CARD, existing weak-to-strong studies typically focus on improving the stronger model itself.
CARD instead incorporates weak-to-strong feedback into adversarial robustness distillation, where improving the teacher is not the final objective but an intermediate step toward further enhancing the compact student's robustness.

\section{Methodology} \label{sec:methodology}
CARD extends conventional one-way adversarial robustness distillation into a closed teacher-student feedback loop.
Given a training set $\mathcal{D}=\{(x_i,y_i)\}_{i=1}^{N}$, a pretrained robust teacher $T^{(0)}$, and an initialized compact student $S^{(0)}$, CARD proceeds through three sequential stages.
The first strong-to-weak stage distills robust knowledge from $T^{(0)}$ to obtain a first-round student $S^{(1)}$.
The student is then frozen and used as a stable feedback source to refine the original teacher, producing $T^{(1)}$.
Finally, the refined teacher provides updated robustness supervision to the same student, yielding the final student $S^{(2)}$. Figure~\ref{fig:framework} shows the complete workflow.

Throughout this section, $\mathcal{L}_{\mathrm{base}}$ denotes the student-training objective of the selected adversarial-distillation method.
CARD retains the corresponding adversarial-example generation procedure and student-distillation objective in both strong-to-weak stages.

\begin{algorithm}[t]
\caption{CARD}
\label{alg:card}

\textbf{Input}:
Pretrained robust teacher \(T^{(0)}\), initialized student \(S^{(0)}\), training set \(\mathcal{D}\), base adversarial-distillation objective \(\mathcal{L}_{\mathrm{base}}\), and perturbation set \(\mathcal{B}_{\epsilon}\).\\
\textbf{Parameters}:
Strong-to-weak distillation epoch \(E_1\), strong-to-weak gain transfer epoch \(E_2\), teacher-refinement epoch \(E_T\), inverse and adversarial optimization steps \(K_{\mathrm{inv}}\) and \(K_{\mathrm{adv}}\), refinement weight \(\lambda\), and teacher learning rate \(\eta_T\).\\
\textbf{Output}:
Refined teacher \(T^{(1)}\) and final student \(S^{(2)}\).

\begin{algorithmic}[1]

\STATE \textbf{Stage 1: Strong-to-weak adversarial distillation}
\STATE Freeze \(T^{(0)}\) and make \(S^{(0)}\) trainable
\STATE
\(S^{(1)} \leftarrow
\operatorname{AdvDistill}
(S^{(0)},T^{(0)},\mathcal{D},E_1;\mathcal{L}_{\mathrm{base}})\)

\STATE \textbf{Stage 2: Weak-to-strong teacher refinement}
\STATE Freeze \(S^{(1)}\) and make \(T^{(0)}\) trainable

\FOR{\(e=1\) \TO \(E_T\)}
    \FOR{each mini-batch \((x,y)\) sampled from \(\mathcal{D}\)}

        \STATE Generate student inverse adversarial examples
        \(x_{\mathrm{inv}}\) using \(K_{\mathrm{inv}}\) steps
        according to Eq.~\eqref{eq: inverse-adversarial-example}

        \STATE Generate discrepancy-maximizing teacher adversarial
        examples \(x_{\mathrm{adv}}\) using
        \(K_{\mathrm{adv}}\) steps according to
        Eq.~\eqref{eq:teacher-student-discrepancy-max}

        \STATE
        \(\mathcal{L}_{\mathrm{rob}}
        \leftarrow
        D_{\mathrm{KL}}
        \bigl(S^{(1)}(x)\,\|\,T^{(1)}(x_{\mathrm{adv}})\bigr)\)

        \STATE
        \(\mathcal{L}_{\mathrm{nat}}
        \leftarrow
        D_{\mathrm{KL}}
        \bigl(S^{(1)}(x_{\mathrm{inv}})\,\|\,T^{(1)}(x)\bigr)\)

        \STATE
        \(\mathcal{L}_{\mathrm{refine}}
        \leftarrow
        \mathcal{L}_{\mathrm{rob}}
        +
        \lambda\mathcal{L}_{\mathrm{nat}}\)

        \STATE
        \(T^{(1)}
        \leftarrow
        T^{(1)}-\eta_T\nabla_{T^{(1)}}\mathcal{L}_{\mathrm{refine}}\)

    \ENDFOR
\ENDFOR

\STATE \textbf{Stage 3: Strong-to-weak gain transfer}
\STATE Freeze \(T^{(1)}\) and make \(S^{(1)}\) trainable
\STATE
\(S^{(2)} \leftarrow
\operatorname{AdvDistill}
(S^{(1)},T^{(1)},\mathcal{D},E_2;\mathcal{L}_{\mathrm{base}})\)

\STATE \textbf{return} \(T^{(1)}\) and \(S^{(2)}\)

\end{algorithmic}
\end{algorithm}

\subsection{Strong-to-Weak Adversarial Distillation}
\label{subsec:strong-to-weak-distillation}

CARD begins with conventional adversarial robustness distillation.
The pretrained teacher $T^{(0)}$ is frozen, and the student is optimized using the objective and adversarial-example generation procedure of the selected base method:
\begin{equation}
\label{eq:student-gain-transfer}
S^{(1)}
\leftarrow
\operatorname{AdvDistill}
\left(
S^{(0)},
T^{(0)},
\mathcal{D},
E_1;
\mathcal{L}_{\mathrm{base}}
\right),
\end{equation}
where $\operatorname{AdvDistill}\left( S^{(0)}, T^{(0)}, \mathcal{D},  E_1; \mathcal{L}_{\mathrm{base}} \right)$ denotes student training under the selected base adversarial-distillation method for $E_1$ epochs.
In our primary instantiation, 
\begin{equation}
\mathcal{L}_{\mathrm{base}}
=
\mathcal{L}_{\mathrm{AdaAD}}.
\end{equation}

Other adversarial-distillation objectives can be substituted without changing the subsequent CARD workflow.
Importantly, the purpose of this first stage is not only to obtain a compact robust student, but also to produce a task-adapted student that has already absorbed robust knowledge from the teacher and can subsequently serve as a meaningful feedback source.

\subsection{Weak-to-Strong Teacher Refinement}
\label{subsec:teacher-refinement}

After the initial distillation stage, we freeze $S^{(1)}$ and initialize the trainable teacher from $T^{(0)}$.
Although the student remains lower-capacity than the teacher, it has already acquired task-specific robust knowledge through the first strong-to-weak stage.
Freezing the student provides a stable post-distillation reference and separates teacher refinement from the evolving dynamics of student training.

Weak-to-strong refinement must address two related issues.
First, directly transferring arbitrary student predictions may propagate student errors to the teacher.
Second, teacher refinement should focus on inputs where student feedback is more informative.
CARD addresses these issues through inverse adversarial student feedback and teacher-student discrepancy maximization, respectively.

Let $S^{(1)}(x),T^{(1)}(x)\in\Delta^{C-1}$ denote the predictive distributions of the frozen student and the teacher being refined over $C$ classes, with $\Delta^{C-1}$ denoting the $(C-1)$-dimensional probability simplex.
For an input $x\in[0,1]^d$, the admissible $\ell_\infty$ perturbation set is defined as
\begin{equation}
\label{eq:perturbation-set}
\mathcal{B}_{\epsilon}(x)
=
\left\{
\delta:
\|\delta\|_{\infty}\leq\epsilon,\;
x+\delta\in[0,1]^d
\right\},
\end{equation}
where $d$ is the input dimension and $\epsilon$ is the perturbation budget.

\subsubsection{Inverse adversarial student feedback}

Directly using the student's prediction on a clean input may propagate erroneous or poorly aligned supervision to the teacher.
To obtain a more favorable student prediction with respect to the ground-truth label, we construct an inverse adversarial example by minimizing the student's classification loss within the admissible perturbation set:
\begin{equation}
\label{eq: inverse-adversarial-example}
\begin{aligned}
\delta_{\mathrm{inv}}^{\star}
&\in
\operatorname*{arg\,min}_{\delta\in\mathcal{B}_{\epsilon}(x)}
\operatorname{CE}
\left(
S^{(1)}(x+\delta),y
\right),\\
x_{\mathrm{inv}}
&=
x+\delta_{\mathrm{inv}}^{\star}.
\end{aligned}
\end{equation}

In contrast to standard adversarial-example generation, which typically seeks to increase the classification loss, Eq.~\eqref{eq: inverse-adversarial-example} searches for a nearby input that reduces the student's loss with respect to the ground-truth label.
Because $\delta=0$ is feasible, exact inner minimization guarantees
\[
\operatorname{CE}
\left(
S^{(1)}(x_{\mathrm{inv}}),y
\right)
\leq
\operatorname{CE}
\left(
S^{(1)}(x),y
\right).
\]

The resulting prediction $S^{(1)}(x_{\mathrm{inv}})$ is then used to supervise the teacher on the corresponding clean input.
This construction reduces the student's ground-truth classification loss before feedback is transferred, but it does not guarantee that every resulting prediction is correct.
The conditions under which such feedback preserves the ground-truth prediction after teacher alignment are analyzed in Sec.~\ref{subsec:feedback-analysis}.

\subsubsection{Teacher-student discrepancy maximization}

Complementing the inverse adversarial feedback mechanism, CARD further identifies perturbed inputs where the teacher and the frozen student exhibit large predictive discrepancies. Specifically, we search for perturbations that maximize the discrepancy between the student's clean prediction and the teacher's prediction on the perturbed input:

\begin{equation}
\label{eq:teacher-student-discrepancy-max}
\begin{aligned}
\delta_{\mathrm{adv}}^{\star}
&\in
\operatorname*{arg\,max}_{\delta\in\mathcal{B}_{\epsilon}(x)}
D_{\mathrm{KL}}
\left(
S^{(1)}(x)
\,\middle\|\,
T^{(1)}(x+\delta)
\right),\\
x_{\mathrm{adv}}
&=
x+\delta_{\mathrm{adv}}^{\star}.
\end{aligned}
\end{equation}

The resulting $x_{\mathrm{adv}}$ is a discrepancy-maximizing perturbed input relative to the frozen student's clean prediction.

\subsubsection{Teacher-refinement objective}

The two branches provide complementary supervision for teacher refinement on perturbed and clean inputs.
Their corresponding losses are defined as
\begin{equation}
\label{eq:teacher-refinement-components}
\begin{aligned}
\mathcal{L}_{\mathrm{rob}}
&=
D_{\mathrm{KL}}
\left(
S^{(1)}(x)
\,\middle\|\,
T^{(1)}(x_{\mathrm{adv}})
\right),\\
\mathcal{L}_{\mathrm{nat}}
&=
D_{\mathrm{KL}}
\left(
S^{(1)}(x_{\mathrm{inv}})
\,\middle\|\,
T^{(1)}(x)
\right).
\end{aligned}
\end{equation}

The robust branch aligns the teacher prediction at a discrepancy-maximizing perturbed input with the student's clean prediction, whereas the natural branch aligns the teacher prediction on the clean input with the student's inverse adversarial prediction.
The complete refinement objective is
\begin{equation}
\label{eq:teacher-refinement-objective}
\mathcal{L}_{\mathrm{refine}}
=
\mathcal{L}_{\mathrm{rob}}
+
\lambda
\mathcal{L}_{\mathrm{nat}},
\qquad
\lambda\geq0,
\end{equation}
where $\lambda$ controls the relative contribution of clean-input refinement.

Because $S^{(1)}$ remains frozen throughout this stage, its predictions serve as fixed supervision and only the teacher parameters are updated:
\begin{equation}
\label{eq:teacher-update}
\theta_{T^{(1)}}
\leftarrow
\theta_{T^{(1)}}
-
\eta_T
\nabla_{\theta_{T^{(1)}}}
\mathcal{L}_{\mathrm{refine}},
\end{equation}
where $\eta_T$ denotes the teacher learning rate.

Starting from $T^{(0)}$, this optimization produces the refined teacher $T^{(1)}$. 
The natural and robust branches respectively refine the teacher on clean inputs and discrepancy-selected perturbed inputs, while the frozen student provides a stable source of feedback.

\subsection{Strong-to-Weak Gain Transfer}
\label{subsec:gain-transfer}

Teacher refinement is not the final objective of CARD.
Its purpose is to provide an updated robust teacher that can further supervise the compact student.
After refinement, we therefore freeze $T^{(1)}$ and resume training from the first-round student $S^{(1)}$, rather than initializing a new student.
Continuing from $S^{(1)}$ preserves the knowledge acquired during the initial distillation stage and allows the refined teacher supervision to be incorporated into the existing student model.

CARD retains the student-training objective and adversarial-example generation procedure of the selected base adversarial-distillation method:
\begin{equation}
\label{eq:student-gain-transfer}
S^{(2)}
\leftarrow
\operatorname{AdvDistill}
\left(
S^{(1)},
T^{(1)},
\mathcal{D},
E_2;
\mathcal{L}_{\mathrm{base}}
\right).
\end{equation}

Compared with the initial strong-to-weak stage, the original teacher $T^{(0)}$ is replaced by the refined teacher $T^{(1)}$, and student optimization resumes from $S^{(1)}$.
The base student-distillation objective remains unchanged.
The resulting $S^{(2)}$ constitutes the final student of CARD, completing the teacher--student feedback loop by transferring the refined teacher supervision back to the same student.
The complete procedure is summarized in Algorithm~\ref{alg:card}.

\subsection{Reliability Analysis of Weak-to-Strong Feedback}
\label{subsec:feedback-analysis}
Having presented the complete CARD workflow, we now analyze the conditions under which weak-to-strong feedback preserves prediction correctness during teacher refinement.
Let \(q,p\in\Delta^{C-1}\) denote a frozen student-feedback distribution and the corresponding teacher output after alignment, respectively.
For the ground-truth class \(y\), we define the student-feedback margin as
\begin{equation}
m(q,y)
=
q_y-\max_{c\neq y}q_c.
\end{equation}
A positive margin \(m(q,y)>0\) indicates that the student feedback predicts the ground-truth class \(y\).

\begin{theorem}[Weak-to-strong prediction preservation]
\label{thm:feedback-transfer}
\leavevmode\\
If \(m(q,y)>0\) and
\begin{equation}
\label{eq:feedback-transfer}
D_{\mathrm{KL}}(q\|p)
<
\frac{m(q,y)^2}{2},
\end{equation}
then
\begin{equation}
p_y>\max_{c\neq y}p_c.
\end{equation}
\end{theorem}

\begin{proof}
Define the total variation distance as
\[
D_{\mathrm{TV}}(p,q)
=
\frac{1}{2}
\sum_{j=1}^{C}
|p_j-q_j|.
\]
For probability distributions $p$ and $q$, each coordinate satisfies
\[
|p_j-q_j|
\leq
D_{\mathrm{TV}}(p,q).
\]

\paragraph{Part (i): Preservation under a positive margin.}
\leavevmode\\
For any class $c\neq y$,
\[
\begin{aligned}
p_y-p_c
&=
(q_y-q_c)
+
(p_y-q_y)
-
(p_c-q_c)\\
&\geq
q_y-q_c
-
|p_y-q_y|
-
|p_c-q_c|\\
&\geq
m(q,y)
-
2D_{\mathrm{TV}}(p,q).
\end{aligned}
\]
Therefore, if
$D_{\mathrm{TV}}(p,q)<m(q,y)/2$,
then $p_y-p_c>0$ for every $c\neq y$.

\paragraph{Part (ii): KL-based control of output deviation.}
\leavevmode\\
Pinsker's inequality gives
\[
D_{\mathrm{TV}}(p,q)
=
D_{\mathrm{TV}}(q,p)
\leq
\sqrt{
\frac{1}{2}
D_{\mathrm{KL}}(q\|p)
}.
\]
Using Eq.~\eqref{eq:feedback-transfer},
\[
D_{\mathrm{TV}}(p,q)
<
\sqrt{
\frac{1}{2}
\cdot
\frac{m(q,y)^2}{2}
}
=
\frac{m(q,y)}{2}.
\]
Part (i) then gives
$p_y-p_c>0$
for every $c\neq y$.
Hence,
$p_y>\max_{c\neq y}p_c$.
\end{proof}

Theorem~\ref{thm:feedback-transfer} is independent of the relative capacities of the student and teacher.
It does not imply that the student is globally superior to the teacher.
Rather, it shows that on samples for which the student feedback predicts the ground-truth class with a positive margin, teacher alignment preserves that prediction whenever the remaining KL misalignment is sufficiently small.

For the natural branch, the theorem applies with
\[
q
=
S^{(1)}(x_{\mathrm{inv}}),
\qquad
p
=
T^{(1)}(x),
\]
where the inverse adversarial optimization reduces the student's classification loss before its prediction is used to supervise the teacher on the clean input.

For the robust branch,
\[
q
=
S^{(1)}(x),
\qquad
p
=
T^{(1)}(x_{\mathrm{adv}}).
\]
Here, discrepancy maximization identifies the perturbed input, while outer minimization aligns the teacher prediction at that input with the clean student reference. We next derive a distribution-level risk bound for weak-to-strong teacher alignment.

\begin{theorem}[Margin-dependent risk bound]
\label{thm:feedback-risk}
\textbf{\leavevmode\\}
Let
$Q,P\in\Delta^{C-1}$
and
$Y\in\{1,\ldots,C\}$
denote jointly distributed student feedback, teacher output after alignment, and ground-truth label, respectively.
Define
\begin{equation}
M_Q
=
Q_Y-\max_{c\neq Y}Q_c,
\end{equation}
and the student-feedback and teacher risks as
\begin{equation}
R_Q
=
\Pr(M_Q\leq0),
\qquad
R_P
=
\Pr
\left(
P_Y
\leq
\max_{c\neq Y}P_c
\right).
\end{equation}
Furthermore, define
\begin{equation}
\psi(Q,P,Y)
=
\begin{cases}
\displaystyle
\min
\left\{
1,\,
\frac{
2D_{\mathrm{KL}}(Q\|P)
}{
M_Q^2
}
\right\},
& M_Q>0,\\[8pt]
0,
& M_Q\leq0.
\end{cases}
\end{equation}
Then,
\begin{equation}
\label{eq:feedback-risk}
\begin{aligned}
R_P
&\leq
R_Q
+
\Pr
\left(
M_Q>0,\,
D_{\mathrm{KL}}(Q\|P)
\geq
\frac{M_Q^2}{2}
\right)\\
&\leq
R_Q
+
\mathbb{E}
\left[
\psi(Q,P,Y)
\right]\\
&\leq
1.
\end{aligned}
\end{equation}
\end{theorem}

\begin{proof}
Let
\[
\mathcal{E}_P
=
\left\{
P_Y
\leq
\max_{c\neq Y}P_c
\right\}
\]
denote the teacher error event.

\paragraph{Part (i): Error decomposition.}
We partition the teacher error event according to whether the student feedback has a positive margin:
\[
\begin{aligned}
R_P
&=
\Pr(\mathcal{E}_P)\\
&=
\Pr
\left(
\mathcal{E}_P
\cap
\{M_Q\leq0\}
\right)
+
\Pr
\left(
\mathcal{E}_P
\cap
\{M_Q>0\}
\right)\\
&\leq
R_Q
+
\Pr
\left(
\mathcal{E}_P
\cap
\{M_Q>0\}
\right).
\end{aligned}
\]

\paragraph{Part (ii): Teacher errors require margin-dependent KL misalignment.}
For a sample satisfying $M_Q>0$, if
\[
D_{\mathrm{KL}}(Q\|P)
<
\frac{M_Q^2}{2},
\]
Theorem~\ref{thm:feedback-transfer} shows that, when the student feedback predicts the ground-truth class with a positive margin and the KL divergence after alignment is sufficiently small, the teacher preserves the same ground-truth prediction.
Therefore,
\[
\mathcal{E}_P
\cap
\{M_Q>0\}
\subseteq
\left\{
M_Q>0,\,
D_{\mathrm{KL}}(Q\|P)
\geq
\frac{M_Q^2}{2}
\right\}.
\]
It follows that
\[
R_P
\leq
R_Q
+
\Pr
\left(
M_Q>0,\,
D_{\mathrm{KL}}(Q\|P)
\geq
\frac{M_Q^2}{2}
\right).
\]

\paragraph{Part (iii): Distribution-level control.}
Let
\[
\mathcal{A}
=
\left\{
M_Q>0,\,
D_{\mathrm{KL}}(Q\|P)
\geq
\frac{M_Q^2}{2}
\right\}.
\]
By the definition of $\psi(Q,P,Y)$,
\[
\mathcal{A}
=
\{\psi(Q,P,Y)=1\}.
\]
Because $\psi(Q,P,Y)$ is non-negative,
\[
\begin{aligned}
\Pr(\mathcal{A})
&=
\Pr
\left(
\psi(Q,P,Y)\geq1
\right)\\
&\leq
\mathbb{E}
\left[
\psi(Q,P,Y)
\right].
\end{aligned}
\]
Combining this result with Part (ii) gives
\[
R_P
\leq
R_Q
+
\mathbb{E}
\left[
\psi(Q,P,Y)
\right].
\]

Moreover,
\[
0
\leq
\psi(Q,P,Y)
\leq
\mathbb{I}_{\{M_Q>0\}},
\]
and therefore
\[
\mathbb{E}
\left[
\psi(Q,P,Y)
\right]
\leq
\Pr(M_Q>0)
=
1-R_Q.
\]
Hence,
\[
R_P
\leq
R_Q
+
\mathbb{E}
\left[
\psi(Q,P,Y)
\right]
\leq
1.
\]
This proves Eq.~\eqref{eq:feedback-risk}.
\end{proof}

For the natural branch,
Theorem~\ref{thm:feedback-risk} is instantiated with
\[
Q
=
S^{(1)}(X_{\mathrm{inv}}),
\qquad
P
=
T^{(1)}(X),
\]
whereas for the robust branch,
\[
Q
=
S^{(1)}(X),
\qquad
P
=
T^{(1)}(X_{\mathrm{adv}}).
\]

The bound shows that the teacher risk is controlled by two factors: errors already present in the student feedback and teacher-student misalignment relative to the feedback margin.
For correctly classified student feedback, a larger positive margin or a smaller KL divergence reduces the margin-normalized misalignment term in the bound.

Theorem~\ref{thm:feedback-risk} does not, by itself, guarantee that teacher refinement always reduces classification risk.
Rather, the bound certifies a strict improvement over the pre-refinement teacher whenever the resulting upper bound is lower than the pre-refinement teacher risk evaluated on the same input distribution.
The analysis is therefore distribution-specific.
For the robust branch, it characterizes the discrepancy-selected perturbed inputs considered during refinement, rather than worst-case robustness over the entire perturbation set.

\begin{table*}[t]
\caption{Clean and white-box robust accuracy (\%) on CIFAR-10 with WRN-28-10 as the teacher and different student architectures.
Within each student-architecture group, the best and second-best results in each column are highlighted in bold and underlined, respectively.
}
\label{tab:cifar10-wrn28-different-students}
\centering
\begin{tabular}{cclccccc}
\toprule
\multirow[c]{2}{*}{\textbf{Type}}
& \multirow[c]{2}{*}{\textbf{Architecture}}
& \multirow[c]{2}{*}{\textbf{Method}}
& \multicolumn{5}{c}{\textbf{CIFAR-10}} \\
\cmidrule(lr){4-8}
& & &
\textbf{Clean}
& \textbf{FGSM}
& \textbf{PGD-20}
& \textbf{CW-20}
& \textbf{AA} \\
\midrule

\textbf{Teacher}
& WRN-28-10
& ST~\cite{stat2023iclr}
& 84.57 & 62.17 & 57.67 & 54.99 & 53.77 \\

\midrule

\multirow[c]{23}{*}{\textbf{Student}}
& \multirow[c]{12}{*}{ResNet-18}
& PGD-AT~\cite{atpgd_ICLRmadry2018}
& 82.50 & 57.52 & 51.53 & 50.57 & 48.04 \\

& & TRADES~\cite{at_tradesicml2019}
& 82.17 & 58.20 & 52.21 & 50.48 & 49.18 \\

& & ST~\cite{stat2023iclr}
& 83.43 & 59.91 & 53.72 & 51.66 & 50.01 \\

& & ARD~\cite{AD_ard2020aaai}
& 83.43 & 59.27 & 53.04 & 51.20 & 49.49 \\

& & RSLAD~\cite{AD_rslad2021iccv}
& 82.44 & 60.99 & 55.66 & 53.19 & 51.90 \\

& & IAD~\cite{AD_iad2022iclr}
& 83.21 & 59.88 & 53.31 & 50.88 & 49.22 \\

& & AdaAD~\cite{AdaAD2023cpvr}
& 84.05 & 60.51 & 55.82 & 53.01 & 51.90 \\

& & PeerAiD~\cite{AD_peeraid2024cvpr}
& 82.75 & 59.68 & 53.78 & 51.64 & 49.56 \\

& & AdaAD+IGDM~\cite{AD_IGDM_2025ICLR}
& 82.81 & 61.04 & \underline{57.01}
& \underline{53.83} & \underline{52.75} \\

\cmidrule(lr){3-8}

& & \textbf{Ours-AdaAD (Once)}
& \underline{84.15} & 61.24 & 55.89 & 53.59 & 52.26 \\

& & \textbf{Ours-AdaAD (Twice)}
& \textbf{84.39} & \underline{61.81} & 55.62 & 53.68 & 52.22 \\

& & \textbf{Ours-AdaAD+IGDM (Once)}
& 82.92 & \textbf{61.95} & \textbf{57.03}
& \textbf{54.37} & \textbf{52.92} \\

\cmidrule(lr){2-8}

& \multirow[c]{11}{*}{MobileNet-V2}
& PGD-AT~\cite{atpgd_ICLRmadry2018}
& 79.99 & 54.94 & 49.23 & 47.70 & 45.11 \\

& & TRADES~\cite{at_tradesicml2019}
& 80.98 & 55.52 & 49.97 & 47.76 & 46.51 \\

& & ST~\cite{stat2023iclr}
& 80.05 & 56.75 & 51.97 & 48.14 & 47.12 \\

& & ARD~\cite{AD_ard2020aaai}
& 82.19 & 58.24 & 52.37 & 49.99 & 48.34 \\

& & RSLAD~\cite{AD_rslad2021iccv}
& 82.65 & \textbf{59.74} & 54.57
& \underline{51.95} & \underline{50.49} \\

& & IAD~\cite{AD_iad2022iclr}
& 81.46 & 58.06 & 52.78 & 50.16 & 48.43 \\

& & AdaAD~\cite{AdaAD2023cpvr}
& \underline{83.34} & 59.21 & 54.03 & 50.96 & 49.57 \\

& & PeerAiD~\cite{AD_peeraid2024cvpr}
& 79.98 & 57.48 & 53.89 & 50.85 & 49.04 \\

& & AdaAD+IGDM~\cite{AD_IGDM_2025ICLR}
& 81.39 & 58.99 & \underline{54.60} & 51.45 & 50.20 \\

\cmidrule(lr){3-8}

& & \textbf{Ours-AdaAD (Once)}
& \textbf{83.37} & \underline{59.59} & 53.89 & 51.85 & 50.36 \\

& & \textbf{Ours-AdaAD+IGDM (Once)}
& 79.71 & 59.04 & \textbf{54.70}
& \textbf{52.00} & \textbf{50.69} \\

\bottomrule
\end{tabular}
\end{table*}

\begin{table*}[t!]
\caption{
Clean and white-box robust accuracy (\%) on CIFAR-10 with
WRN-34-10 as the teacher and different student architectures.
Within each student-architecture group, the best and second-best results in each column are highlighted in bold and underlined, respectively.
}
\label{tab:cifar10-wrn34-different-students}
\centering
\begin{tabular}{cclccccc}
\toprule
\multirow[c]{2}{*}{\textbf{Type}}
& \multirow[c]{2}{*}{\textbf{Architecture}}
& \multirow[c]{2}{*}{\textbf{Method}}
& \multicolumn{5}{c}{\textbf{CIFAR-10}} \\
\cmidrule(lr){4-8}
& & &
\textbf{Clean}
& \textbf{FGSM}
& \textbf{PGD-20}
& \textbf{CW-20}
& \textbf{AA} \\
\midrule

\textbf{Teacher}
& WRN-34-10
& ST~\cite{stat2023iclr}
& 84.75 & 62.40 & 57.83 & 55.49 & 54.33 \\

\midrule

\multirow[c]{23}{*}{\textbf{Student}}
& \multirow[c]{12}{*}{ResNet-18}
& PGD-AT~\cite{atpgd_ICLRmadry2018}
& 82.50 & 57.52 & 51.53 & 50.57 & 48.04 \\

& & TRADES~\cite{at_tradesicml2019}
& 82.17 & 58.20 & 52.21 & 50.48 & 49.18 \\

& & ST~\cite{stat2023iclr}
& 83.43 & 59.91 & 53.72 & 51.66 & 50.01 \\

& & ARD~\cite{AD_ard2020aaai}
& 83.29 & 59.32 & 53.65 & 51.63 & 50.01 \\

& & RSLAD~\cite{AD_rslad2021iccv}
& 82.98 & 60.71 & 55.61 & 53.00 & 51.49 \\

& & IAD~\cite{AD_iad2022iclr}
& 82.84 & 59.69 & 54.13 & 51.72 & 50.10 \\

& & AdaAD~\cite{AdaAD2023cpvr}
& 84.28 & 60.97 & 56.17 & 53.36 & 52.02 \\

& & PeerAiD~\cite{AD_peeraid2024cvpr}
& 82.63 & 59.87 & 54.21 & 52.24 & 50.25 \\

& & AdaAD+IGDM~\cite{AD_IGDM_2025ICLR}
& 83.78 & \underline{61.54} & \underline{56.64} & \underline{54.09} & \underline{52.65} \\

\cmidrule(lr){3-8}

& & \textbf{Ours-AdaAD (Once)}
& \underline{84.45} & 61.07 & 55.68
& 53.63 & 52.22 \\

& & \textbf{Ours-AdaAD (Twice)}
& \textbf{84.49} & 61.49 & 55.39
& 53.95 & 52.34 \\

& & \textbf{Ours-AdaAD+IGDM (Once)}
& 84.06 & \textbf{61.89} & \textbf{56.73}
& \textbf{54.50} & \textbf{53.10} \\

\cmidrule(lr){2-8}

& \multirow[c]{11}{*}{MobileNet-V2}
& PGD-AT~\cite{atpgd_ICLRmadry2018}
& 79.99 & 54.94 & 49.23 & 47.70 & 45.11 \\

& & TRADES~\cite{at_tradesicml2019}
& 80.98 & 55.52 & 49.97 & 47.76 & 46.51 \\

& & ST~\cite{stat2023iclr}
& 80.05 & 56.75 & 51.97 & 48.14 & 47.12 \\

& & ARD~\cite{AD_ard2020aaai}
& 81.80 & 57.77 & 52.50 & 50.22 & 48.67 \\

& & RSLAD~\cite{AD_rslad2021iccv}
& 83.11 & \textbf{59.66} & 54.24
& \underline{51.68} & 50.28 \\

& & IAD~\cite{AD_iad2022iclr}
& 81.44 & 58.47 & 52.47 & 50.15 & 48.06 \\

& & AdaAD~\cite{AdaAD2023cpvr}
& \underline{83.86} & 59.32 & 54.06 & 51.41 & 49.95 \\

& & PeerAiD~\cite{AD_peeraid2024cvpr}
& 82.12 & 59.25 & 53.70 & 50.76 & 48.76 \\

& & AdaAD+IGDM~\cite{AD_IGDM_2025ICLR}
& 81.48 & 58.83 & \textbf{54.72} & 51.52
& \underline{50.34} \\

\cmidrule(lr){3-8}

& & \textbf{Ours-AdaAD (Once)}
& \textbf{83.90} & 59.55 & 53.71 & 51.57 & 50.02 \\

& & \textbf{Ours-AdaAD+IGDM (Once)}
& 81.73 & \underline{59.58} & \underline{54.54}
& \textbf{51.84} & \textbf{50.37} \\

\bottomrule
\end{tabular}

\end{table*}

\begin{table*}[t!]
\caption{Clean and white-box robust accuracy (\%) on CIFAR-100 with WRN-34-10 as the teacher and different student architectures.
Within each student-architecture group, the best and second-best results in each column are highlighted in bold and underlined, respectively.}
\label{tab:cifar100-wrn34-different-students}
\centering
\begin{tabular}{cclccccc}
\toprule
\multirow[c]{2}{*}{\textbf{Type}}
& \multirow[c]{2}{*}{\textbf{Architecture}}
& \multirow[c]{2}{*}{\textbf{Method}}
& \multicolumn{5}{c}{\textbf{CIFAR-100}} \\
\cmidrule(lr){4-8}
& & &
\textbf{Clean}
& \textbf{FGSM}
& \textbf{PGD-20}
& \textbf{CW-20}
& \textbf{AA} \\
\midrule

\textbf{Teacher}
& WRN-34-10
& ST~\cite{stat2023iclr}
& 59.40 & 35.92 & 33.71 & 30.30 & 29.19 \\

\midrule

\multirow[c]{23}{*}{\textbf{Student}}
& \multirow[c]{12}{*}{ResNet-18}
& PGD-AT~\cite{atpgd_ICLRmadry2018}
& 57.52 & 32.12 & 28.50 & 27.12 & 24.68 \\

& & TRADES~\cite{at_tradesicml2019}
& 56.32 & 29.83 & 26.95 & 23.77 & 22.69 \\

& & ST~\cite{stat2023iclr}
& 55.75 & 33.00 & 30.92 & 26.89 & 26.04 \\

& & ARD~\cite{AD_ard2020aaai}
& 58.44 & 34.92 & 31.54 & 28.47 & 26.42 \\

& & RSLAD~\cite{AD_rslad2021iccv}
& 56.55 & 35.08 & 32.85 & 29.34 & 28.00 \\

& & IAD~\cite{AD_iad2022iclr}
& 57.82 & 34.96 & 31.94 & 28.32 & 26.80 \\

& & AdaAD~\cite{AdaAD2023cpvr}
& 58.91 & 34.99 & 32.57 & 29.00 & 27.80 \\

& & PeerAiD~\cite{AD_peeraid2024cvpr}
& 56.76 & 35.68 & 32.99 & 29.66 & 27.97 \\

& & AdaAD+IGDM~\cite{AD_IGDM_2025ICLR}
& 58.49 & 35.30 & \underline{33.14} & 29.48 & 28.17 \\

\cmidrule(lr){3-8}

& & \textbf{Ours-AdaAD (Once)}
& \underline{59.66} & 35.75 & 33.03
& \underline{29.75} & \underline{28.19} \\

& & \textbf{Ours-AdaAD (Twice)}
& \textbf{59.68} & \underline{35.77} & 32.95
& 29.62 & \underline{28.19} \\

& & \textbf{Ours-AdaAD+IGDM (Once)}
& 58.62 & \textbf{36.00} & \textbf{33.63}
& \textbf{30.15} & \textbf{29.00} \\

\cmidrule(lr){2-8}

& \multirow[c]{11}{*}{MobileNet-V2}
& PGD-AT~\cite{atpgd_ICLRmadry2018}
& 56.00 & 30.99 & 27.73 & 25.84 & 23.25 \\

& & TRADES~\cite{at_tradesicml2019}
& 56.82 & 30.79 & 28.22 & 24.16 & 23.18 \\

& & ST~\cite{stat2023iclr}
& 54.53 & 31.91 & 30.10 & 25.41 & 24.64 \\

& & ARD~\cite{AD_ard2020aaai}
& 57.25 & 33.82 & 31.00 & 27.90 & 26.02 \\

& & RSLAD~\cite{AD_rslad2021iccv}
& 56.78 & 34.50 & 32.11 & 28.43 & 26.94 \\

& & IAD~\cite{AD_iad2022iclr}
& 57.41 & \underline{34.67} & 31.59 & 28.09 & 26.36 \\

& & AdaAD~\cite{AdaAD2023cpvr}
& \underline{57.95} & 33.41 & 30.92 & 27.00 & 25.55 \\

& & PeerAiD~\cite{AD_peeraid2024cvpr}
& 56.67 & 34.59 & 31.47 & \underline{28.82} & 26.92 \\

& & AdaAD+IGDM~\cite{AD_IGDM_2025ICLR}
& 57.12 & 34.34 & \underline{32.33} & 28.47 & \underline{27.44} \\

\cmidrule(lr){3-8}

& & \textbf{Ours-AdaAD (Once)}
& \textbf{58.24} & 33.98 & 31.30 & 28.11 & 26.40 \\

& & \textbf{Ours-AdaAD+IGDM (Once)}
& 57.14 & \textbf{34.82} & \textbf{32.59}
& \textbf{28.88} & \textbf{27.74} \\

\bottomrule
\end{tabular}
\end{table*}

\begin{table*}[t!]
\caption{
Clean and white-box robust accuracy (\%) on TinyImageNet-10 with WRN-34-10 as the teacher and ResNet-18 as the student. 
The best and second-best student results in each column are highlighted in bold and underlined, respectively.
}
\label{tab:tinyimagenet10-wrn34-rn18}
\centering
\begin{tabular}{cclccccc}
\toprule
\multirow[c]{2}{*}{\textbf{Type}}
& \multirow[c]{2}{*}{\textbf{Architecture}}
& \multirow[c]{2}{*}{\textbf{Method}}
& \multicolumn{5}{c}{\textbf{TinyImageNet-10}} \\
\cmidrule(lr){4-8}
& & &
\textbf{Clean}
& \textbf{FGSM}
& \textbf{PGD-20}
& \textbf{CW-20}
& \textbf{AA} \\
\midrule

\textbf{Teacher}
& WRN-34-10
& ST~\cite{stat2023iclr}
& 77.80 & 62.00 & 60.80 & 60.00 & 60.00 \\

\midrule

\multirow[c]{12}{*}{\textbf{Student}}
& \multirow[c]{12}{*}{ResNet-18}
& PGD-AT~\cite{atpgd_ICLRmadry2018}
& 79.60 & 61.00 & 57.20 & 57.20 & 55.20 \\

& & TRADES~\cite{at_tradesicml2019}
& 77.80 & 58.80 & 56.80 & 55.60 & 55.40 \\

& & ST~\cite{stat2023iclr}
& 76.60 & 62.20 & 61.00 & 59.00 & 58.20 \\

& & ARD~\cite{AD_ard2020aaai}
& 79.60 & 60.40 & 58.60 & 57.00 & 56.40 \\

& & RSLAD~\cite{AD_rslad2021iccv}
& 76.80 & 62.60 & 61.80 & 59.20 & 59.20 \\

& & IAD~\cite{AD_iad2022iclr}
& 79.20 & 61.40 & 59.40 & 57.60 & 57.20 \\

& & AdaAD~\cite{AdaAD2023cpvr}
& 78.40 & \underline{62.80} & \underline{62.40}
& 59.60 & 59.40 \\

& & PeerAiD~\cite{AD_peeraid2024cvpr}
& \underline{82.60} & 62.40 & 58.80 & 58.20 & 57.20 \\

& & AdaAD+IGDM~\cite{AD_IGDM_2025ICLR}
& 77.40 & \underline{62.80} & 62.20
& \underline{60.20} & \underline{59.80} \\

\cmidrule(lr){3-8}

& & \textbf{Ours-AdaAD (Once)}
& 79.20 & 61.60 & 61.00 & 60.00 & 59.40 \\

& & \textbf{Ours-AdaAD+IGDM (Once)}
& 78.80 & \textbf{63.80} & \textbf{62.60}
& \textbf{61.00} & \textbf{60.80} \\

& & \textbf{Ours-ARD (Once)}
& \textbf{82.80} & 60.80 & 57.60 & 57.00 & 55.40 \\

\bottomrule
\end{tabular}

\end{table*}

\begin{table*}[t!]
\caption{Transfer-based black-box accuracy (\%) on CIFAR-100 with WRN-34-10 as the teacher and ResNet-18 as the target student.
PGD-100 and CW-100 adversarial examples are generated on independently trained robust ResNet-18, VGG-19, and MobileNet-V2 surrogate models.
The best and second-best results in each column are shown in bold and underlined, respectively.}
\label{tab:blackbox-cifar100-wrn34-rn18}
\centering
\begin{tabular}{lcccccc}
\toprule
\multirow[c]{2}{*}{\textbf{Method}}
& \multicolumn{2}{c}{\textbf{ResNet-18}}
& \multicolumn{2}{c}{\textbf{VGG-19}}
& \multicolumn{2}{c}{\textbf{MobileNet-V2}} \\
\cmidrule(lr){2-3}
\cmidrule(lr){4-5}
\cmidrule(lr){6-7}
& \textbf{PGD-100} & \textbf{CW-100}
& \textbf{PGD-100} & \textbf{CW-100}
& \textbf{PGD-100} & \textbf{CW-100} \\
\midrule

PGD-AT~\cite{atpgd_ICLRmadry2018}
& 33.69 & 33.41
& 41.64 & 42.03
& 38.36 & 38.27 \\

TRADES~\cite{at_tradesicml2019}
& 35.71 & 36.49
& 40.75 & 41.68
& 37.55 & 37.91 \\

ST~\cite{stat2023iclr}
& 37.67 & 38.17
& 42.00 & 42.68
& 39.13 & 39.65 \\

ARD~\cite{AD_ard2020aaai}
& 39.32 & 39.31
& 44.19 & 44.24
& 40.73 & 40.68 \\

RSLAD~\cite{AD_rslad2021iccv}
& 39.18 & 39.72
& 43.23 & 43.86
& 40.58 & 40.81 \\

IAD~\cite{AD_iad2022iclr}
& 39.57 & 39.80
& 44.39 & 44.77
& 41.36 & 41.61 \\

AdaAD~\cite{AdaAD2023cpvr}
& 39.64 & 39.76
& 44.76 & 44.85
& 41.50 & 41.11 \\

PeerAiD~\cite{AD_peeraid2024cvpr}
& 39.53 & 39.97
& 43.70 & 43.39
& 41.42 & 41.00 \\

AdaAD+IGDM~\cite{AD_IGDM_2025ICLR}
& 39.76 & 40.05
& 44.86 & 45.31
& 41.51 & 41.40 \\

\midrule

\textbf{Ours-AdaAD (Once)}
& 40.14 & 40.23
& \underline{45.32} & \underline{45.53}
& 41.94 & 41.78 \\

\textbf{Ours-AdaAD (Twice)}
& \underline{40.16} & \textbf{40.44}
& \textbf{45.50} & \textbf{45.87}
& \textbf{42.16} & \textbf{42.02} \\

\textbf{Ours-AdaAD+IGDM (Once)}
& \textbf{40.20} & \underline{40.28}
& 45.14 & 45.48
& \underline{41.98} & \underline{41.82} \\

\bottomrule
\end{tabular}
\end{table*}

\begin{table*}[t]
\caption{
Transfer-based black-box accuracy (\%) on CIFAR-100 with WRN-34-10 as the teacher and ResNet-18 as the target student. 
PGD-100 and CW-100 adversarial examples are generated on independently trained natural ResNet-18, MobileNet-V2, and VGG-19 surrogate models. The best and second-best results in each column are shown in bold and underlined, respectively.
}
\label{tab:blackbox-clean-surrogate-cifar100-part1}
\centering
\begin{tabular}{lcccccc}
\toprule
\multirow[c]{2}{*}{\textbf{Method}}
& \multicolumn{2}{c}{\textbf{ResNet-18}}
& \multicolumn{2}{c}{\textbf{MobileNet-V2}}
& \multicolumn{2}{c}{\textbf{VGG-19}}
 \\
\cmidrule(lr){2-3}
\cmidrule(lr){4-5}
\cmidrule(lr){6-7}

& \textbf{PGD-100} & \textbf{CW-100}
& \textbf{PGD-100} & \textbf{CW-100}
& \textbf{PGD-100} & \textbf{CW-100}
 \\
\midrule

PGD-AT~\cite{atpgd_ICLRmadry2018}
& 56.29 & 56.31
& 56.30 & 56.32
& 56.49 & 56.53
\\

TRADES~\cite{at_tradesicml2019}
& 54.83 & 54.96
& 54.80 & 54.97
& 55.12 & 55.28
\\

ST~\cite{stat2023iclr}
& 54.43 & 54.54
& 54.35 & 54.31
& 54.50 & 54.68
\\

ARD~\cite{AD_ard2020aaai}
& 56.99 & 57.20
& 57.12 & 57.04
& 57.38 & 57.15
\\

RSLAD~\cite{AD_rslad2021iccv}
& 55.16 & 55.24
& 55.14 & 55.16
& 55.46 & 55.50
\\

IAD~\cite{AD_iad2022iclr}
& 56.81 & 56.98
& 56.82 & 56.91
& 57.04 & 57.09
\\

AdaAD~\cite{AdaAD2023cpvr}
& 57.36 & 57.50
& 57.36 & 57.47
& 57.53 & 57.53
\\

PeerAiD~\cite{AD_peeraid2024cvpr}
& 55.57 & 55.54
& 55.71 & 55.73
& 55.72 & 55.64
\\

AdaAD+IGDM~\cite{AD_IGDM_2025ICLR}
& 57.12 & 57.07
& 56.98 & 56.98
& 57.11 & 57.20
\\

\midrule

\textbf{Ours-AdaAD (Once)}
& \underline{57.98} & \underline{58.11}
& \underline{57.96} & \underline{58.10}
& \underline{58.12} & \underline{58.23}
\\

\textbf{Ours-AdaAD (Twice)}
& \textbf{57.99} & \textbf{58.18}
& \textbf{58.01} & \textbf{58.15}
& \textbf{58.37} & \textbf{58.44}
\\

\bottomrule
\end{tabular}

\end{table*}

\begin{table*}[t]
\caption{
Transfer-based black-box accuracy (\%) on CIFAR-100 with WRN-34-10 as the teacher and ResNet-18 as the target student. 
PGD-100 and CW-100 adversarial examples are generated on independently trained natural ResNet-50, DenseNet-169, and WRN-34-10 surrogate models. 
The best and second-best results in each column are shown in bold and underlined, respectively.
}
\label{tab:blackbox-clean-surrogate-cifar100-part2}
\centering
\begin{tabular}{lcccccc}
\toprule
\multirow[c]{2}{*}{\textbf{Method}}
& \multicolumn{2}{c}{\textbf{ResNet-50}}
& \multicolumn{2}{c}{\textbf{DenseNet-169}}
& \multicolumn{2}{c}{\textbf{WRN-34-10}} \\
\cmidrule(lr){2-3}
\cmidrule(lr){4-5}
\cmidrule(lr){6-7}
& \textbf{PGD-100} & \textbf{CW-100}
& \textbf{PGD-100} & \textbf{CW-100}
& \textbf{PGD-100} & \textbf{CW-100} \\
\midrule

PGD-AT~\cite{atpgd_ICLRmadry2018}
& 55.99 & 55.92
& 56.26 & 56.26
& 57.03 & 57.07 \\

TRADES~\cite{at_tradesicml2019}
& 54.75 & 54.50
& 54.68 & 54.89
& 55.54 & 55.64 \\

ST~\cite{stat2023iclr}
& 54.13 & 54.32
& 54.48 & 54.42
& 54.94 & 54.99 \\

ARD~\cite{AD_ard2020aaai}
& 56.85 & 56.93
& 57.18 & 56.96
& 57.53 & 57.60 \\

RSLAD~\cite{AD_rslad2021iccv}
& 54.74 & 54.97
& 55.04 & 55.25
& 55.71 & 55.75 \\

IAD~\cite{AD_iad2022iclr}
& 56.70 & 56.81
& 56.75 & 56.93
& 57.32 & 57.30 \\

AdaAD~\cite{AdaAD2023cpvr}
& 57.26 & 57.21
& 57.27 & 57.47
& 58.17 & 58.03 \\

PeerAiD~\cite{AD_peeraid2024cvpr}
& 55.31 & 55.36
& 55.60 & 55.57
& 56.16 & 56.24 \\

AdaAD+IGDM~\cite{AD_IGDM_2025ICLR}
& 56.68 & 56.84
& 57.00 & 57.15
& 57.53 & 57.54 \\

\midrule

\textbf{Ours-AdaAD (Once)}
& \underline{57.96} & \textbf{58.03}
& \underline{58.04} & \textbf{58.07}
& \underline{58.81} & \underline{58.75} \\

\textbf{Ours-AdaAD (Twice)}
& \textbf{58.05} & \underline{57.96}
& \textbf{58.30} & \underline{58.04}
& \textbf{58.82} & \textbf{58.82} \\

\bottomrule
\end{tabular}

\end{table*}

\section{Experiments}
We evaluate CARD through white-box and transfer-based black-box robustness experiments, followed by ablation studies, hyperparameter analysis, and qualitative visualization.
The experiments are designed to examine the effectiveness and generalizability of CARD across different datasets, student architectures, and base adversarial-distillation methods.

\subsection{Experimental Setup}

This subsection describes the datasets, network architectures, compared methods, training configurations, and evaluation protocols used throughout the experiments.
Unless otherwise stated, a baseline method and its corresponding CARD variant use the same dataset split, model architecture, and base adversarial-distillation configuration.

\subsubsection{Datasets and preprocessing.}

We evaluate CARD on CIFAR-10, CIFAR-100~\citep{cifar100_2009}, and a fixed ten-class subset of TinyImageNet~\citep{tinyimagenet}, denoted as TinyImageNet-10.
CIFAR-10 and CIFAR-100 each contain 50,000 training images and 10,000 test images at a spatial resolution of \(32\times32\) pixels, with 10 and 100 classes, respectively.
We use the original training and test partitions without introducing additional training data.

TinyImageNet-10 consists of ten fixed classes selected from TinyImageNet.
For each class, all 500 training images and 50 validation images are retained, resulting in 5,000 training images and 500 evaluation images at a spatial resolution of \(64\times64\) pixels.
The original TinyImageNet training and validation splits are used for training and evaluation, respectively.
The selected classes are Goldfish, King penguin, American alligator, Bullet train, School bus, Broom, Bucket, Hourglass, Snorkel, and Espresso.
The same TinyImageNet-10 subset is used in all experiments.

For training, we apply random cropping with a padding of four pixels and random horizontal flipping, following commonly adopted preprocessing protocols in adversarial training~\citep{stat2023iclr} and adversarial distillation~\citep{AD_IGDM_2025ICLR,AdaAD2023cpvr}.
No data augmentation is applied during evaluation.

\subsubsection{Architectures and baselines.}

We use WideResNet-28-10 (WRN-28-10) and WideResNet-34-10 (WRN-34-10)~\citep{wrnarxiv2016} as pretrained robust teachers.
ResNet-18~\citep{resnet18cvpr2016} and MobileNet-V2 ~\citep{mnv2cvpr2018} are used as student architectures.

We compare CARD with representative adversarial training methods, including
PGD-AT~\citep{atpgd_ICLRmadry2018},
TRADES~\citep{at_tradesicml2019},
and ST~\citep{stat2023iclr},
as well as adversarial robustness distillation methods, including
ARD~\citep{AD_ard2020aaai},
RSLAD~\citep{AD_rslad2021iccv},
IAD~\citep{AD_iad2022iclr},
AdaAD~\citep{AdaAD2023cpvr},
PeerAiD~\citep{AD_peeraid2024cvpr},
and AdaAD+IGDM~\citep{AD_IGDM_2025ICLR}.
The compared methods are reproduced following their published training protocols and hyperparameter settings whenever applicable.

For each CARD variant, the initial strong-to-weak stage follows the corresponding base adversarial-distillation method.
CARD directly starts teacher refinement from the student checkpoint produced by this initial distillation stage, without introducing additional student training before teacher refinement.

In the main experiments, CARD is instantiated on AdaAD and AdaAD+IGDM.
\textit{Once} denotes one complete teacher-refinement and gain-transfer cycle.
For variants marked \textit{Twice}, the same refinement and gain-transfer procedure is applied for a second cycle, initialized from the teacher and student obtained after the preceding cycle.

\subsubsection{Training and implementation details.}

Teacher refinement is initialized from the pretrained robust teacher \(T^{(0)}\), while the student obtained from the initial strong-to-weak distillation stage, \(S^{(1)}\), is kept frozen throughout refinement.
Thus, only the teacher parameters are updated during this stage.

The teacher is optimized using stochastic gradient descent (SGD) with a momentum of \(0.9\), a weight decay of \(2\times10^{-4}\), and a batch size of 128.
Teacher refinement is performed for 20 epochs.
The learning rate is initialized to \(0.02\), reduced to \(0.01\) after the 10th epoch, and further reduced to \(0.001\) after the 18th epoch.
The temperature used in the KL-divergence terms is fixed to
\(\tau=1.0\).
The selection of \(\lambda\) is further analyzed in Section~\ref{sec:lambda-sensitivity}.

For both inverse adversarial feedback and teacher--student discrepancy maximization, perturbations are constrained within an \(\ell_{\infty}\) ball with radius \(\epsilon=8/255\).
Each inner optimization is performed for 10 steps with a step size of \(2/255\).
After each step, the perturbed input is projected back onto the corresponding \(\ell_{\infty}\) ball centered at the clean input and clipped to the valid input range.

After teacher refinement, the refined teacher \(T^{(1)}\) is frozen, and student training resumes from the first-round checkpoint \(S^{(1)}\).
The student is trained for an additional 30 epochs under the supervision of \(T^{(1)}\).
The student-distillation objective and adversarial-example generation procedure follow those of the corresponding base adversarial-distillation method, with method-specific hyperparameters retained unless otherwise stated.

\subsubsection{Evaluation protocol.}

For white-box evaluation, we consider FGSM~\citep{fgsmiclr2015},
PGD-20~\citep{atpgd_ICLRmadry2018},
a 20-step attack based on the CW margin loss~\citep{att_cwsp2017_qianyi},
and AutoAttack~\citep{att_aaicml20}.
All attacks are constrained by an \(\ell_{\infty}\) perturbation budget of \(\epsilon=8/255\).
For PGD-20 and the CW-20 attack, the step size is \(2/255\).
Each white-box attack is generated directly against the evaluated target model, and the same attack configuration is used for a baseline and its corresponding CARD variant. 
AutoAttack is evaluated in its standard configuration, comprising APGD-CE, APGD-DLR, FAB, and Square Attack, under the same \(\ell_{\infty}\) perturbation budget.

For transfer-based black-box evaluation, adversarial examples are first generated against independently trained surrogate models and then transferred directly to the target student without access to the target model parameters or gradients during attack generation.
We consider PGD-100 and CW-100, both with \(\epsilon=8/255\) and a step size of \(2/255\).

The robust surrogate models include ResNet-18, VGG-19~\citep{vgg_iclr2015}, and MobileNet-V2, whereas the naturally trained surrogate models include ResNet-18, MobileNet-V2, VGG-19, ResNet-50, DenseNet-169~\citep{densenet_cvpr2017}, and WRN-34-10.
For each surrogate model and attack configuration, the same transferred adversarial examples are used to evaluate all target methods, ensuring that the baseline methods and their corresponding CARD variants are compared under identical transfer attacks.

\subsection{White-Box Robustness}

We first evaluate CARD under white-box attacks across different datasets, teacher architectures, student architectures, and base adversarial-distillation objectives.

\subsubsection{Results on CIFAR-10.}

Table~\ref{tab:cifar10-wrn28-different-students} reports the CIFAR-10 results with WRN-28-10 as the teacher.
For ResNet-18, applying CARD once to AdaAD+IGDM improves clean, FGSM, PGD-20, CW-20, and AA accuracy over the original AdaAD+IGDM baseline by \(0.11\), \(0.91\), \(0.02\), \(0.54\), and \(0.17\) percentage points, respectively.
Similarly, applying CARD once to AdaAD improves all five metrics over the original AdaAD baseline.
A second CARD cycle further increases clean, FGSM, and CW-20 accuracy relative to the one-cycle AdaAD-based variant, while slightly reducing PGD-20 and AA accuracy.

For MobileNet-V2, the effect of CARD is more mixed.
Compared with the original AdaAD+IGDM baseline, applying CARD once increases FGSM, PGD-20, CW-20, and AA accuracy by \(0.05\), \(0.10\), \(0.55\), and \(0.49\) percentage points, respectively, while reducing clean accuracy by \(1.68\) percentage points.
Compared with the original AdaAD baseline, applying CARD once increases clean, FGSM, CW-20, and AA accuracy by \(0.03\), \(0.38\), \(0.89\), and \(0.79\) percentage points, respectively, whereas PGD-20 accuracy decreases by \(0.14\) percentage points.
These results indicate that the effect of CARD on CIFAR-10 varies with both the student architecture and the base adversarial-distillation objective.

Table~\ref{tab:cifar10-wrn34-different-students} further reports the CIFAR-10 results with WRN-34-10 as the teacher.
For ResNet-18, compared with the original AdaAD baseline, applying CARD once increases clean, FGSM, CW-20, and AA accuracy by \(0.17\), \(0.10\), \(0.27\), and \(0.20\) percentage points, respectively, while PGD-20 accuracy decreases by \(0.49\) percentage points.
Applying a second CARD cycle further changes the results relative to the one-cycle variant, increasing clean, FGSM, CW-20, and AA accuracy by \(0.04\), \(0.42\), \(0.32\), and \(0.12\) percentage points, respectively, while PGD-20 accuracy decreases by an additional \(0.29\) percentage points.
Compared with the original AdaAD+IGDM baseline, applying CARD once improves clean, FGSM, PGD-20, CW-20, and AA accuracy by \(0.28\), \(0.35\), \(0.09\), \(0.41\), and \(0.45\) percentage points, respectively.
This AdaAD+IGDM-based CARD variant achieves the highest FGSM, PGD-20, CW-20, and AA accuracy among the evaluated ResNet-18 student methods.

For MobileNet-V2, the effect of CARD is again metric dependent.
Compared with the original AdaAD baseline, applying CARD once increases clean, FGSM, CW-20, and AA accuracy by \(0.04\), \(0.23\), \(0.16\), and \(0.07\) percentage points, respectively, while PGD-20 accuracy decreases by \(0.35\) percentage points.
Compared with the original AdaAD+IGDM baseline, applying CARD once increases clean, FGSM, CW-20, and AA accuracy by \(0.25\), \(0.75\), \(0.32\), and \(0.03\) percentage points, respectively, while PGD-20 accuracy decreases by \(0.18\) percentage points.
Among the evaluated MobileNet-V2 student methods, the AdaAD+IGDM-based CARD variant achieves the highest CW-20 and AA accuracy.

\subsubsection{Results on CIFAR-100.}

Table~\ref{tab:cifar100-wrn34-different-students} reports the results on CIFAR-100 with WRN-34-10 as the teacher. 
For ResNet-18, compared with the original AdaAD+IGDM baseline, applying CARD once improves clean, FGSM, PGD-20, CW-20, and AA accuracy by \(0.13\), \(0.70\), \(0.49\), \(0.67\), and \(0.83\) percentage points, respectively.
For MobileNet-V2, compared with the corresponding AdaAD+IGDM baseline, the improvements are \(0.02\), \(0.48\), \(0.26\), \(0.41\), and \(0.30\) percentage points, respectively. 
CARD also consistently improves AdaAD after one cycle for both student architectures.
For ResNet-18, compared with the original AdaAD baseline, applying CARD once increases clean, FGSM, PGD-20, CW-20, and AA accuracy by \(0.75\), \(0.76\), \(0.46\), \(0.75\), and \(0.39\) percentage points, respectively.
For MobileNet-V2, the corresponding improvements over the original AdaAD baseline are \(0.29\), \(0.57\), \(0.38\), \(1.11\), and \(0.85\) percentage points.
These results show that one CARD cycle consistently improves both clean and robust accuracy over the corresponding AdaAD and AdaAD+IGDM baselines across the two evaluated student architectures.

For ResNet-18 with AdaAD, applying a second CARD cycle yields only marginal changes relative to the one-cycle variant.
Specifically, clean and FGSM accuracy increase by \(0.02\) percentage points each, PGD-20 and CW-20 accuracy decrease by \(0.08\) and \(0.13\) percentage points, respectively, and AA accuracy remains unchanged.
This result indicates that additional refinement-transfer cycles do not necessarily provide uniform improvements across all evaluation metrics.

\subsubsection{Results on TinyImageNet-10.}

Table~\ref{tab:tinyimagenet10-wrn34-rn18} reports the results on TinyImageNet-10 with WRN-34-10 as the teacher and ResNet-18 as the student.
Compared with the original AdaAD+IGDM baseline, applying CARD once improves clean, FGSM, PGD-20, CW-20, and AA accuracy by \(1.40\), \(1.00\), \(0.40\), \(0.80\), and \(1.00\) percentage points, respectively.
The resulting AdaAD+IGDM-based CARD variant achieves the highest robust accuracy under all four attacks among the evaluated student methods, while also improving clean accuracy over its corresponding baseline.

The effects of CARD are less uniform when combined with the other distillation objectives.
Compared with the original AdaAD baseline, applying CARD once increases clean and CW-20 accuracy by \(0.80\) and \(0.40\) percentage points, respectively, and leaves AA accuracy unchanged, while FGSM and PGD-20 accuracy decrease by \(1.20\) and \(1.40\) percentage points.
Compared with the original ARD baseline, applying CARD once increases clean and FGSM accuracy by \(3.20\) and \(0.40\) percentage points, respectively, leaves CW-20 accuracy unchanged, and decreases PGD-20 and AA accuracy by \(1.00\) percentage point each.
The ARD-based CARD variant achieves the highest clean accuracy among the evaluated student methods.

Overall, the effect of CARD on TinyImageNet-10 varies with the base adversarial-distillation objective.
The AdaAD +IGDM-based variant improves both clean accuracy and robustness under all four evaluated attacks, whereas the AdaAD- and ARD-based variants exhibit metric-dependent trade-offs between clean and robust performance.

\subsection{Black-Box Transfer Robustness}

We further evaluate the transfer-based black-box robustness of CARD on CIFAR-100 using independently trained robust and naturally trained surrogate models.

\subsubsection{Results with Robust Surrogates.}

Table~\ref{tab:blackbox-cifar100-wrn34-rn18} reports the transfer-based black-box results with ResNet-18 as the target student and ResNet-18, VGG-19, and MobileNet-V2 as robust surrogate models.
Compared with the original AdaAD baseline, applying CARD twice improves transferred-attack accuracy by \(0.52\)--\(1.02\) percentage points across all six surrogate-attack combinations.

Compared with the original AdaAD+IGDM baseline, applying CARD once also improves transferred-attack accuracy across all six surrogate-attack combinations, with gains ranging from \(0.17\) to \(0.47\) percentage points.
These results indicate that CARD consistently improves transfer-based black-box robustness against attacks generated from multiple robust surrogate architectures.

\subsubsection{Results with Natural Surrogates.}

Tables~\ref{tab:blackbox-clean-surrogate-cifar100-part1} and~\ref{tab:blackbox-clean-surrogate-cifar100-part2} report the transfer-based black-box results with ResNet-18 as the target student and ResNet-18, MobileNet-V2, VGG-19, ResNet-50, DenseNet-169, and WRN-34-10 as naturally trained surrogate models.
PGD-100 and CW-100 adversarial examples generated on each surrogate are directly transferred to the target student.

Compared with the original AdaAD baseline, applying CARD once improves transferred-attack accuracy by \(0.59\)--\(0.82\) percentage points across all 12 surrogate-attack combinations.
Applying CARD twice also improves all 12 combinations over the original AdaAD baseline, with gains ranging from \(0.57\) to \(1.03\) percentage points.

Overall, these results demonstrate consistent improvements in transfer-based black-box robustness across attacks generated from multiple robust and naturally trained surrogate architectures.

\subsection{Ablation Studies}
\label{subsec:ablation-studies}

We conduct ablation studies to examine three aspects of CARD:
(1) whether weak-to-strong teacher refinement contributes beyond additional student training,
(2) whether teacher refinement and student gain transfer should be optimized jointly or sequentially, and
(3) how different supervision signals affect teacher refinement.
Unless otherwise specified, the student is initialized from the checkpoint obtained after the initial strong-to-weak adversarial distillation stage, whereas the teacher is initialized from the pretrained robust teacher.
All teacher-refinement variants use an \(\ell_{\infty}\) perturbation budget of \(\epsilon=8/255\), with 10 inner optimization steps and a step size of \(2/255\).

\subsubsection{Contribution of the Cyclic Components}

We first examine whether the improvement of CARD can be attributed simply to additional student training.
The baseline corresponds to the original AdaAD student.
The \emph{S2W-only} variant removes weak-to-strong (W2S) teacher refinement and directly continues strong-to-weak (S2W) student training under the guidance of the unchanged teacher using the original AdaAD objective.
The complete CARD variant performs W2S teacher refinement before the subsequent S2W gain-transfer stage.

As shown in Table~\ref{tab:ablation_w2s_component}, compared with the original AdaAD baseline, S2W-only training improves CW-20 and AA accuracy by \(0.55\) and \(0.60\) percentage points, respectively, but reduces clean accuracy by \(0.86\) percentage points.
Compared with S2W-only training, adding W2S teacher refinement further improves clean, CW-20, and AA accuracy by \(1.15\), \(0.56\), and \(0.25\) percentage points, respectively.
Consequently, the complete CARD framework outperforms the original AdaAD baseline by \(0.29\) percentage points in clean accuracy and by \(1.11\) and \(0.85\) percentage points under CW-20 and AA, respectively.
These results indicate that the improvement of CARD cannot be explained solely by additional student training and that W2S teacher refinement provides an additional contribution.

\begin{table}[t]
\caption{Component ablation of CARD on CIFAR-100 with WRN-34-10 as the teacher and MobileNet-V2 as the student.
W2S and S2W denote weak-to-strong teacher refinement and strong-to-weak gain transfer, respectively.}
\label{tab:ablation_w2s_component}
\setlength{\tabcolsep}{1.9mm}
\centering
\begin{tabular}{lccccc}
\toprule
\textbf{Variant}
& \textbf{W2S}
& \textbf{S2W}
& \textbf{Clean}
& \textbf{CW-20}
& \textbf{AA} \\
\midrule
AdaAD
&
&
& 57.95
& 27.00
& 25.55 \\

S2W Only
&
& $\checkmark$
& 57.09
& 27.55
& 26.15 \\

\textbf{CARD}
& $\checkmark$
& $\checkmark$
& \textbf{58.24}
& \textbf{28.11}
& \textbf{26.40} \\
\bottomrule
\end{tabular}
\end{table}

\subsubsection{Separated versus Joint Optimization}

We next examine whether W2S teacher refinement and S2W student gain transfer should be optimized jointly or sequentially.
CARD adopts separated optimization, in which teacher refinement is completed before the refined teacher guides the subsequent student gain-transfer stage. 

For comparison, we construct a \emph{Joint} variant that updates the teacher and student simultaneously within a single training stage.
The teacher is optimized using the CARD refinement objective, whereas the student is optimized using the original AdaAD objective:
\begin{equation}
\label{eq:abl-joint-training}
\min_{\theta_T}\mathcal{L}_{\mathrm{refine}},
\qquad
\min_{\theta_S}\mathcal{L}_{\mathrm{AdaAD}},
\end{equation}
where \(\theta_T\) and \(\theta_S\) denote the teacher and student parameters, respectively.
During joint optimization, the current student provides feedback for teacher refinement while the current teacher simultaneously guides student optimization.

Both models are trained for 30 epochs.
The teacher learning rate is set to \(0.02\), \(0.01\), \(0.001\), and \(0.0001\) during epochs 0--7, 8--15, 16--24, and 25--29, respectively, while the student learning rate is set to one-tenth of the corresponding teacher learning rate.

As shown in Table~\ref{tab:ablation_separated_training}, compared with joint optimization, separated optimization improves clean accuracy by \(1.31\) percentage points and improves PGD-20, CW-20, and AA accuracy by \(0.07\), \(0.11\), and \(0.10\) percentage points, respectively.
Although the robustness improvements are modest, they are consistent across all three attacks, while the improvement in clean accuracy is more pronounced.
These results support the separated optimization strategy adopted in CARD.

\begin{table}[t]
\caption{Comparison of joint and separated optimization on CIFAR-100 with WRN-34-10 as the teacher and MobileNet-V2 as the student.
Both variants contain the same W2S teacher-refinement and S2W gain-transfer components.}
\label{tab:ablation_separated_training}
\setlength{\tabcolsep}{1.7mm}
\centering
\begin{tabular}{lcccc}
\toprule
\textbf{Method}
& \textbf{Clean}
& \textbf{PGD-20}
& \textbf{CW-20}
& \textbf{AA} \\
\midrule

AdaAD
& 57.95
& 30.92
& 27.00
& 25.55 \\

\midrule

Joint
& 56.93
& 31.23
& 28.00
& 26.30 \\

\textbf{Separated (Ours)}
& \textbf{58.24}
& \textbf{31.30}
& \textbf{28.11}
& \textbf{26.40} \\

\bottomrule
\end{tabular}
\end{table}

\subsubsection{Teacher-Refinement Signals}

We further investigate how different supervision signals affect teacher refinement.
During refinement, the student is frozen and only the teacher is updated.
After each refinement variant, the resulting teacher is frozen and the same student undergoes the subsequent S2W gain-transfer stage using the AdaAD objective.
Therefore, the variants differ only in the supervision used for teacher refinement. 
The \emph{w/o Refinement} variant directly performs additional student training using the unchanged teacher.
\emph{GT Only} removes model-prediction supervision and relies only on ground-truth labels.
\emph{Teacher Clean + GT} uses the teacher's clean prediction for robust refinement and the ground-truth label for clean-input refinement, thereby removing student feedback.
\emph{Student Clean} uses the clean student prediction for both refinement branches and therefore removes inverse adversarial feedback.
\emph{Student Inverse} instead uses the inverse-adversarial student prediction for both branches.
\emph{Teacher Inverse} replaces student feedback with the corresponding inverse-adversarial teacher prediction.
The complete CARD objective combines the clean student prediction for robust refinement with the inverse-adversarial student prediction for natural refinement.

For the prediction-based variants, the perturbed teacher input is constructed by maximizing the discrepancy between the corresponding reference prediction and the teacher prediction on the perturbed input.
For \emph{GT Only}, the perturbed input is instead generated by maximizing the teacher's cross-entropy loss.
The inverse variants construct the corresponding inverse adversarial input by minimizing the ground-truth cross-entropy loss of the student or teacher within the same perturbation set.
These definitions isolate the effects of student feedback, inverse adversarial feedback, and the choice of supervision signal while keeping the subsequent student gain-transfer stage unchanged.

As shown in Table~\ref{tab:ablation_refinement_signals_cifar100}, none of the alternative refinement-signal variants improves all reported metrics over the original AdaAD baseline.
Although \emph{GT Only} achieves the highest clean accuracy of \(61.85\%\), its PGD-20 and AA accuracy decrease by \(6.27\) and \(4.16\) percentage points, respectively, relative to AdaAD.
In contrast, the complete CARD objective achieves the highest robust accuracy under PGD-20, CW-20, and AA.
Compared with the original AdaAD baseline, CARD improves clean, PGD-20, CW-20, and AA accuracy by \(0.75\), \(0.46\), \(0.75\), and \(0.39\) percentage points, respectively.
Compared with the \emph{w/o Refinement} variant, the corresponding improvements are \(1.18\), \(0.24\), \(0.43\), and \(0.27\) percentage points.
These results support the use of clean student feedback for robust refinement together with inverse-adversarial student feedback for natural refinement.

\begin{table}[t]
\caption{Comparison of teacher-refinement supervision signals on CIFAR-100 with WRN-34-10 as the teacher and ResNet-18 as the student.
The reported results correspond to the student obtained after gain transfer from the resulting teacher. 
GT denotes ground-truth hard-label supervision.}
\label{tab:ablation_refinement_signals_cifar100}
\setlength{\tabcolsep}{1.6mm}
\centering
\begin{tabular}{lcccc}
\toprule
\textbf{Refinement Signal}
& \textbf{Clean}
& \textbf{PGD-20}
& \textbf{CW-20}
& \textbf{AA} \\
\midrule

AdaAD
& 58.91
& 32.57
& 29.00
& 27.80 \\

\midrule

w/o Refinement
& 58.48
& 32.79
& 29.32
& 27.92 \\

GT Only
& \textbf{61.85}
& 26.30
& 26.09
& 23.64 \\

Teacher Clean + GT
& 61.75
& 29.67
& 28.14
& 26.49 \\

Student Clean
& 58.76
& 32.81
& 29.22
& 27.96 \\

Student Inverse
& 60.98
& 32.51
& 29.29
& 27.73 \\

Teacher Inverse
& 61.49
& 29.86
& 28.48
& 26.74 \\

\textbf{Ours-AdaAD (Once)}
& 59.66
& \textbf{33.03}
& \textbf{29.75}
& \textbf{28.19} \\

\bottomrule
\end{tabular}
\end{table}

\subsection{Sensitivity to the Refinement Weight \(\lambda\)}
\label{sec:lambda-sensitivity}

The refinement weight \(\lambda\) controls the contribution of the natural-refinement term relative to the robust-refinement term and is therefore the main method-specific hyperparameter in CARD.
We analyze its sensitivity on CIFAR-10 using WRN-28-10 as the teacher and ResNet-18 as the student.
All other training and refinement settings are kept unchanged.
We vary \(\lambda\) separately in the first- and second-cycle refinement stages and evaluate both the refined teacher and the student obtained after the subsequent gain-transfer stage.

\begin{figure*}[t!]
    \centering
    \begin{tabular}{@{}cc@{}}
        \includegraphics[width=0.4\textwidth]{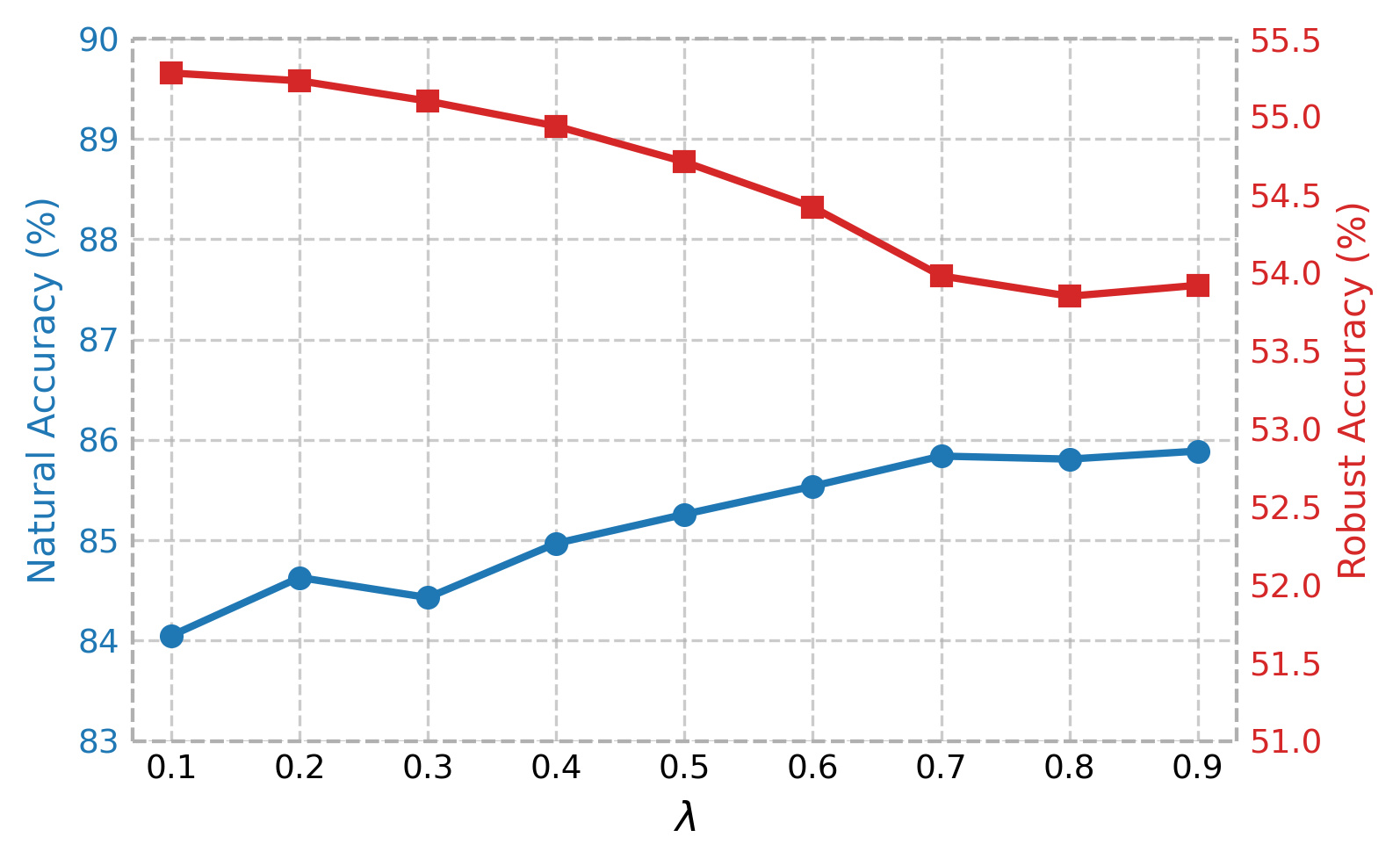} &
        \includegraphics[width=0.4\textwidth]{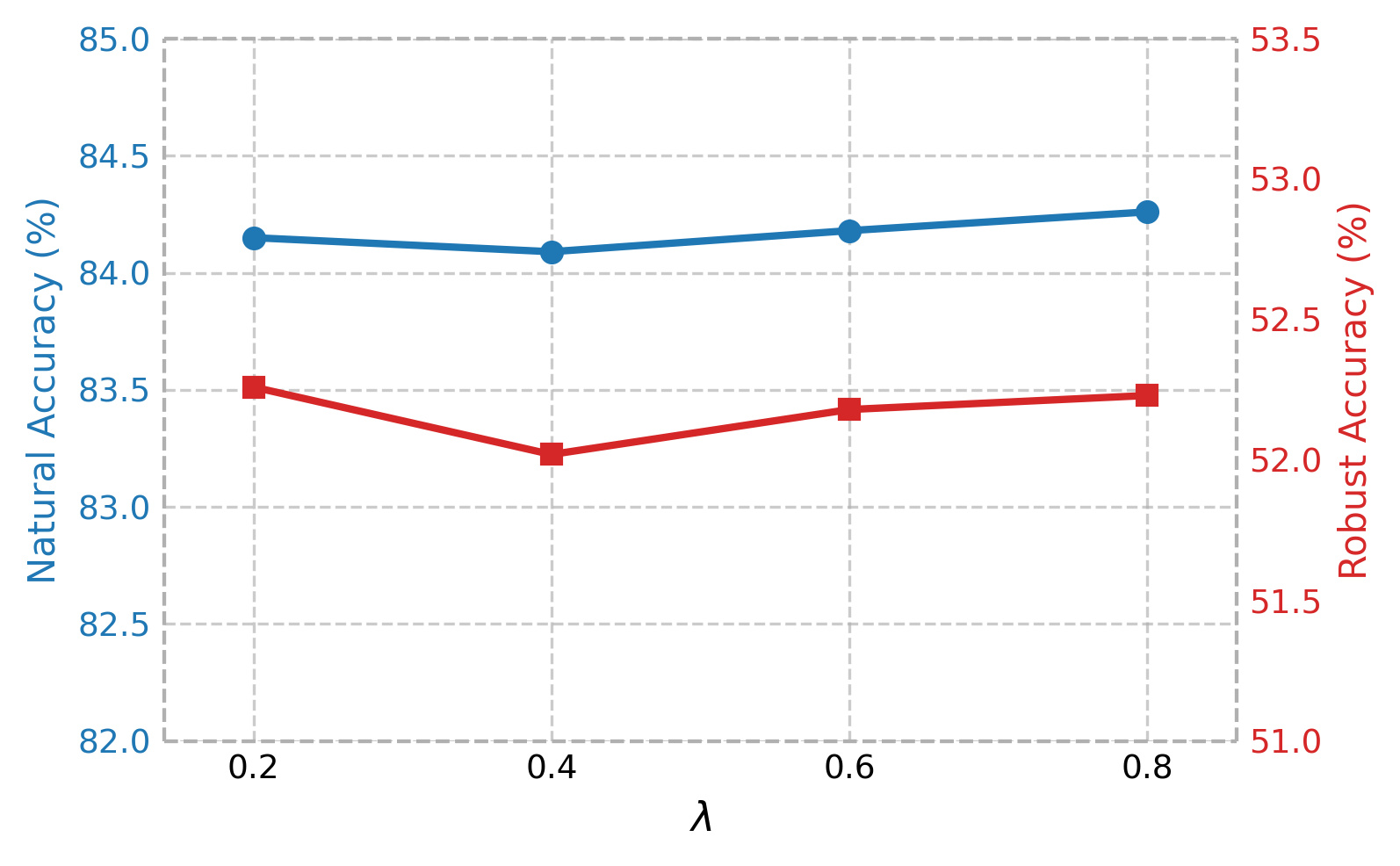} \\
        (a) First-Cycle Teacher &
        (b) First-Cycle Student \\[2pt]
        \includegraphics[width=0.4\textwidth]{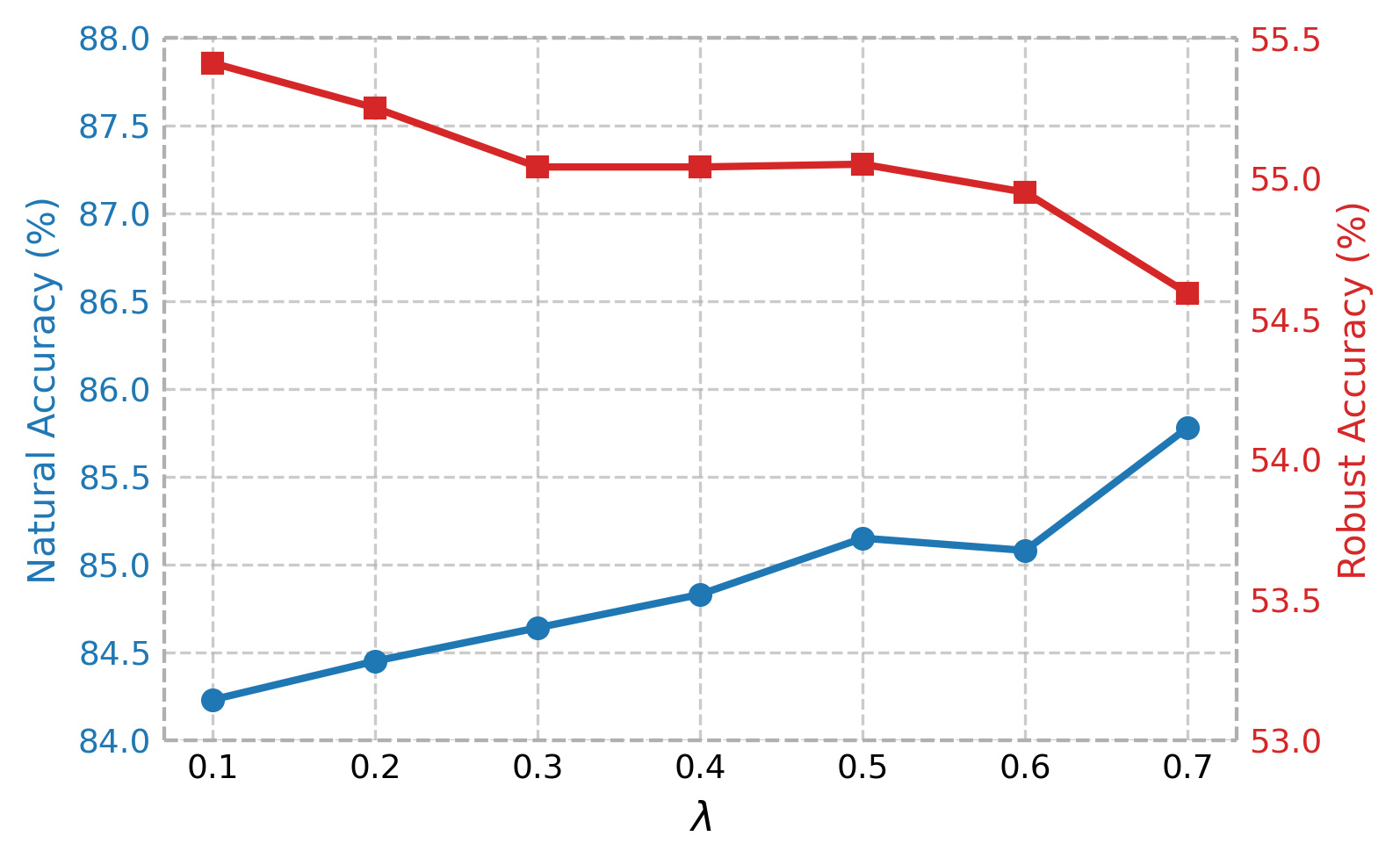} &
        \includegraphics[width=0.4\textwidth]{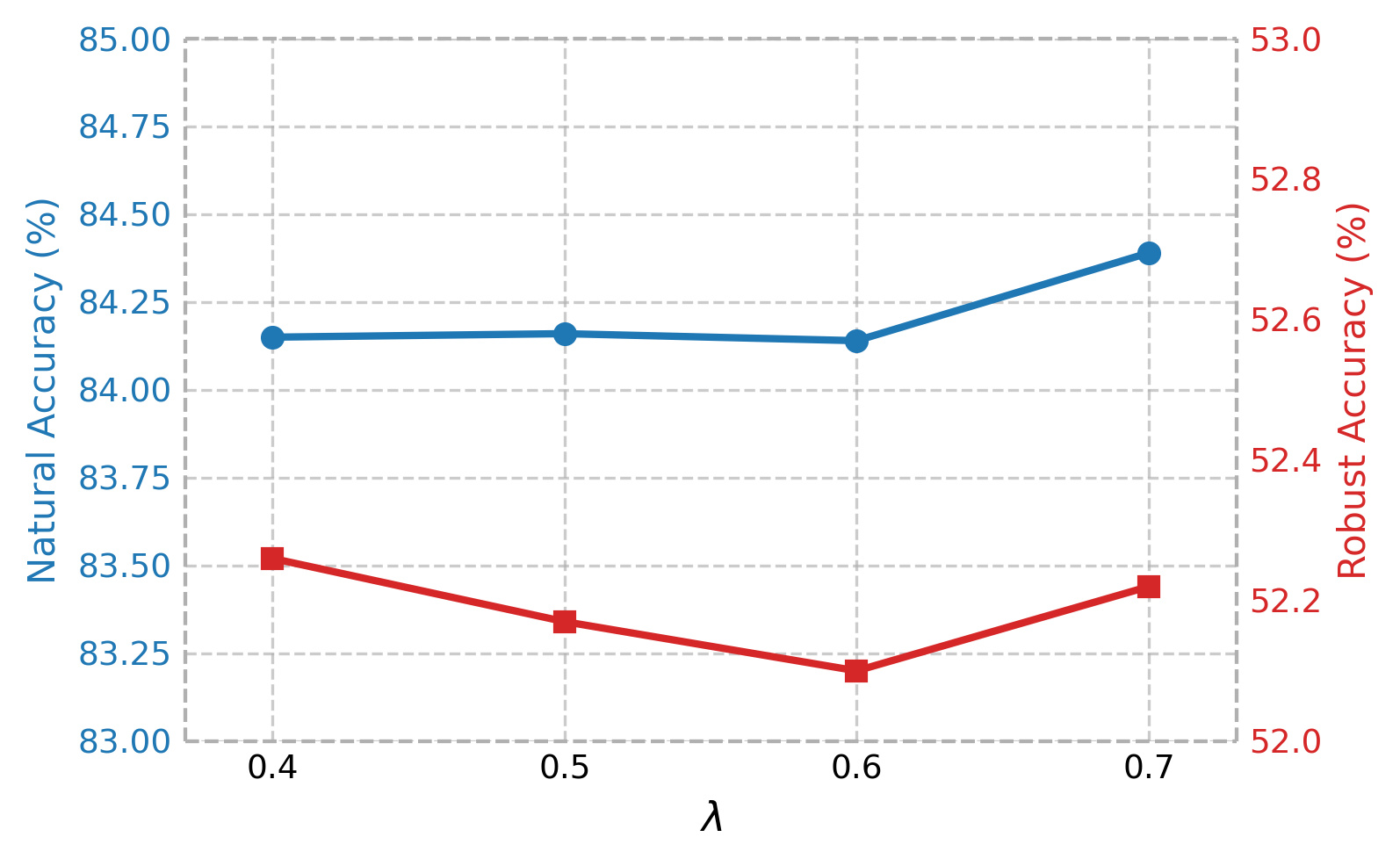} \\
        (c) Second-Cycle Teacher &
        (d) Second-Cycle Student
    \end{tabular}
    \caption{
    Sensitivity of CARD to the refinement weight \(\lambda\) on CIFAR-10 with WRN-28-10 as the teacher and ResNet-18 as the student.
    Blue and red curves denote natural and robust accuracy, respectively.
    Panels (a) and (c) report the refined teacher after the first and second refinement cycles, while panels (b) and (d) report the corresponding students after the subsequent gain-transfer stage.
    }
    \label{fig:lambda_search}
\end{figure*}

As shown in Fig.~\ref{fig:lambda_search}(a), increasing \(\lambda\) generally increases the natural accuracy of the first-cycle refined teacher while reducing its robust accuracy.
A similar tendency is observed for the second-cycle teacher in
Fig.~\ref{fig:lambda_search}(c).
This behavior is consistent with the role of \(\lambda\): a larger value places greater emphasis on natural refinement, whereas a smaller value gives relatively greater importance to the robust-refinement term.

In contrast, the student performance in Figs.~\ref{fig:lambda_search}(b) and~\ref{fig:lambda_search}(d) varies less across the evaluated values of \(\lambda\).
Although the refined teacher exhibits a clearer natural--robust trade-off, the student obtained after the subsequent gain-transfer stage is comparatively less sensitive to the precise refinement weight.
This observation suggests that the effect of \(\lambda\) on the intermediate teacher does not translate directly into an equally large variation in the final student performance.

Based on validation performance, we set \(\lambda=0.2\) for the first-cycle refinement and \(\lambda=0.7\) for the second-cycle refinement.
Within the evaluated range, the final student performance remains relatively stable, indicating limited sensitivity to \(\lambda\) in the considered setting.

\subsection{Qualitative and Mechanistic Analysis}
\label{sec:qualitative-analysis}

We provide complementary qualitative analyses to examine the attention patterns learned by CARD and the prediction behavior underlying weak-to-strong teacher refinement.
Specifically, Grad-CAM visualizations are used to compare discriminative attention under adversarial perturbations, while confusion matrices are used to examine the class consistency of inverse-adversarial feedback and the prediction changes induced by teacher refinement and subsequent gain transfer.

\subsubsection{Grad-CAM Visualizations}
\label{subsec:gradcam}

We use Grad-CAM~\cite{gradcam_iccv2017} to visualize the discriminative regions emphasized by different models on adversarial CIFAR-100 inputs.
As shown in Fig.~\ref{fig:gradcam-adversarial}, several baseline students exhibit relatively fragmented activation patterns or place substantial attention on background regions.
In contrast, the CARD student shows more spatially concentrated activation on object-relevant regions in the visualized examples and exhibits greater qualitative similarity to the robust teacher.

Compared with the baselines, the CARD student attends more consistently to object-relevant regions and shows activation patterns closer to those of the robust teacher under adversarial perturbations.
This qualitative similarity suggests that the student better retains the discriminative patterns of the robust teacher, consistent with the improved robustness observed in the quantitative results.

\begin{figure*}[t!]
    \centering
    \includegraphics[width=0.9\linewidth]{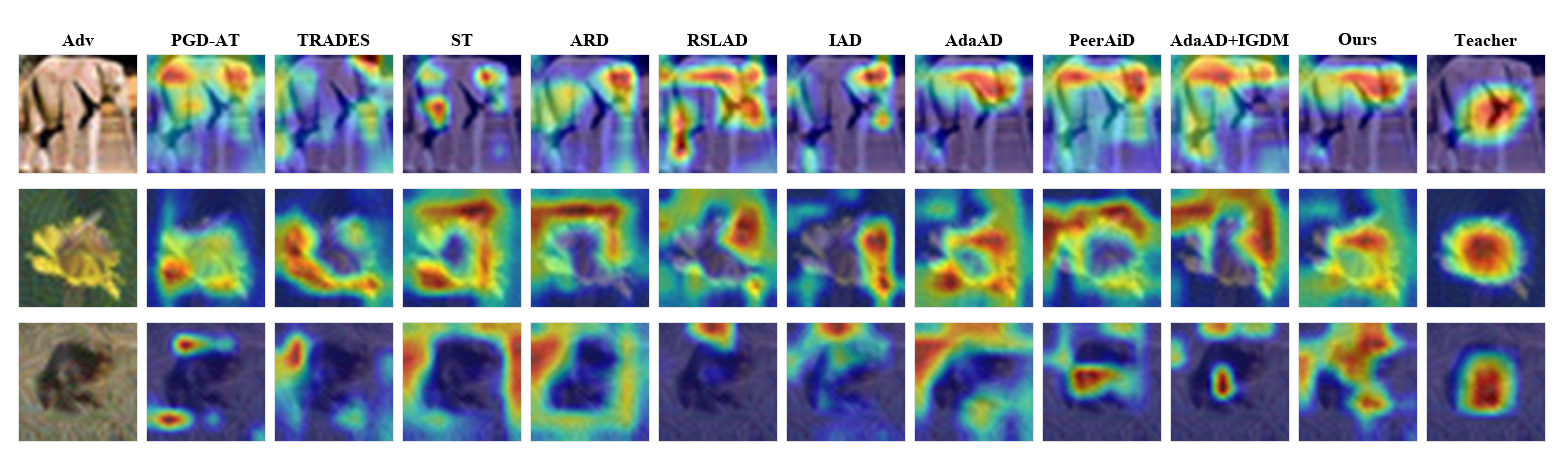}
    \caption{
    Grad-CAM visualizations on adversarial CIFAR-100 inputs.
    The leftmost column shows the corresponding adversarial images, the middle columns show activation maps of different student methods, and the rightmost column shows those of the robust teacher.
    In the visualized examples, the CARD student exhibits more concentrated activation on object-relevant regions and greater qualitative similarity to the robust teacher than the baseline students.
    }
    \label{fig:gradcam-adversarial}
\end{figure*}

\begin{figure*}[t]
    \centering
    \includegraphics[width=0.9\linewidth]{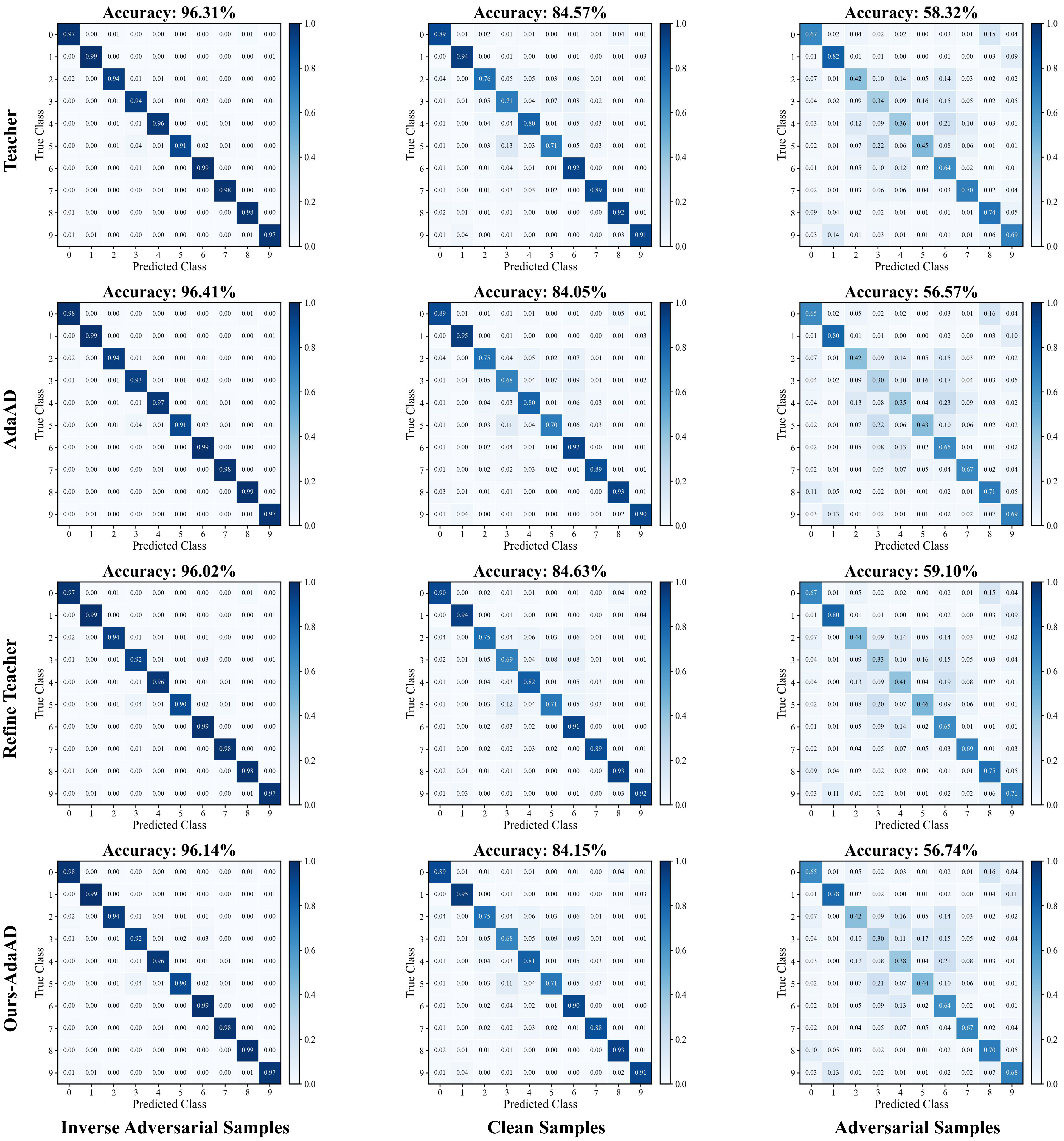}
\caption{
Normalized confusion matrices on inverse adversarial, clean, and adversarial CIFAR-10 inputs with WRN-28-10 as the teacher and ResNet-18 as the student. 
The inverse adversarial and adversarial inputs are generated using 10-step iterative optimization. 
Columns correspond to the three input types, whereas rows show the initial teacher, AdaAD student, refined teacher, and final CARD student. 
The strong diagonal concentration on inverse adversarial inputs indicates that these inputs provide highly reliable class-consistent predictions.
Teacher refinement improves clean and adversarial accuracy, and part of these gains is subsequently transferred to the final student.
}
    \label{fig:confusion-matrix}
\end{figure*}

\subsubsection{Confusion Matrix Analysis}
\label{subsec:app-confusion-matrix}

To examine the feasibility of weak-to-strong teacher refinement, we analyze whether inverse adversarial student predictions provide sufficiently reliable feedback to the teacher.

Figure~\ref{fig:confusion-matrix} presents normalized confusion matrices on inverse adversarial, clean, and adversarial CIFAR-10 inputs, using WRN-28-10 as the teacher and ResNet-18 as the student.
Both perturbed input types are generated using 10-step iterative optimization.

All models achieve over \(96\%\) accuracy on inverse adversarial inputs, with predictions strongly concentrated along the main diagonal. This indicates that inverse adversarial inputs preserve reliable class-consistent information for teacher refinement.
Refinement increases the teacher's clean and adversarial accuracy from \(84.57\%\) and \(58.32\%\) to \(84.63\%\) and \(59.10\%\), respectively. 
The final CARD student subsequently improves AdaAD from \(84.05\%\) to \(84.15\%\) on clean inputs and from \(56.57\%\) to \(56.74\%\) on adversarial inputs.

The refined teacher and final student show slight decreases on inverse adversarial inputs while remaining above \(96\%\). 
Because these perturbations are optimized to facilitate correct classification, the reduced benefit from them, together with improved adversarial accuracy, suggests less sensitivity to input perturbations.
Overall, consistent with Theorems~\ref{thm:feedback-transfer} and~\ref{thm:feedback-risk}, the results provide empirical support that reliable student feedback can be transferred to the teacher when it has a positive prediction margin and is sufficiently aligned with the teacher output.

\section{Conclusion}
In this work, we proposed \emph{Cyclic Adversarial Robust Distillation} (CARD), which extends conventional one-way adversarial robustness distillation into a closed teacher--student feedback loop consisting of strong-to-weak distillation, weak-to-strong teacher refinement, and subsequent gain transfer.
CARD uses inverse-adversarial student feedback and discrepancy-maximizing inputs to refine the teacher while retaining the base student-distillation objective.
We further provided theoretical conditions for reliable weak-to-strong feedback and showed experimentally that CARD improves adversarial robustness in most evaluated settings under both white-box and transfer-based black-box attacks.

CARD still has several limitations.
Its current evaluation is restricted to relatively small-scale image classification benchmarks and primarily considers \(\ell_{\infty}\)-bounded attacks, while the additional refinement and gain-transfer stages introduce extra training cost.
Moreover, the improvement is not uniform across all attacks and base distillation methods.
Future work will investigate more efficient refinement strategies, larger-scale models and datasets, broader threat models, and adaptive feedback mechanisms that can better determine when and how student knowledge should be used to refine the teacher.

\printcredits

\section*{Declaration of competing interest}
The authors declare that they have no known competing financial interests or personal relationships that could have appeared to influence the work reported in this paper.

\section*{Declaration of generative AI and AI-assisted technologies in the manuscript preparation process}

During the preparation of this work, the author(s) used GPT-5.6 Sol for polishing the expression. The author(s) reviewed and edited the output as needed and take full responsibility for the content of the published article.

\section*{Acknowledgments}
This work is supported by the Yunnan Research Project (Grant Nos. 202505AF350053, 202503AG380006, 202401-AT070474, 202501AU070059, 202403AP140021), National Natural Science Foundation of China (Grant Nos. 62562061, 62502422 and 62462067), and Yunnan Provincial Department of Education Science Research Project (Grant Nos. 2025J0006, 2024J0010 and 2025J0007).

\section*{Data availability}
Data will be made available on request.

\bibliographystyle{cas-model2-names}

\bibliography{cas-refs}

\end{document}